\documentclass[sigconf]{acmart}

\usepackage[table]{xcolor}
\usepackage{multirow}
\usepackage{array}
\usepackage{enumitem}
\usepackage{lipsum}
\usepackage{subcaption} 
\usepackage{hyperref}
\usepackage{cleveref}
\crefname{appendix}{Appendix}{Appendices}
\Crefname{appendix}{Appendix}{Appendices}
\usepackage{bm}
\usepackage{fancyvrb}
\input{mintedcache/default.pygstyle}
\definecolor{lightgray}{gray}{0.9}
\definecolor{palegray}{gray}{0.98}

\usepackage{algorithm}
\usepackage{algorithmic}
\usepackage{thmtools, thm-restate}
\usepackage{etoc} 

\newenvironment{sproof}{%
  \renewcommand{\proofname}{Proof Sketch}\proof}{\endproof}

\newcommand{\affKAIST}{\affiliation{%
  \institution{Korea Advanced Institute of Science and Technology}%
  \city{Seongnam}\state{Gyeonggi-do}\country{Republic of Korea}}}
\newcommand{\affINEEJI}{\affiliation{%
  \institution{INEEJI Corp.}%
  \city{Seongnam}\state{Gyeonggi-do}\country{Republic of Korea}}}

\newcommand{\sigcode}{\textcolor{magenta}{\url{https://github.com/leekwoon/sig/}}}
\newcommand{\sigdoi}{\textcolor{magenta}{\url{https://doi.org/10.6084/m9.figshare.32310030}}}

\AtBeginDocument{%
  }

\copyrightyear{2026}
\acmYear{2026}
\setcopyright{cc}
\setcctype{by-nc-nd}
\acmConference[KDD '26]{Proceedings of the 32nd ACM SIGKDD Conference on Knowledge Discovery and Data Mining V.2}{August 09--13, 2026}{Jeju Island, Republic of Korea}
\acmBooktitle{Proceedings of the 32nd ACM SIGKDD Conference on Knowledge Discovery and Data Mining V.2 (KDD '26), August 09--13, 2026, Jeju Island, Republic of Korea}
\acmDOI{10.1145/3770855.3817830}
\acmISBN{979-8-4007-2259-2/2026/08}

\acmSubmissionID{v2rtp1295}

\begin{document}

\title{Spectral Integrated Gradients for Coarse-to-Fine Feature Attribution}





\author{Soyeon Kim}
\orcid{0009-0001-5037-0902}
\affKAIST
\affINEEJI
\email{soyeon.k@kaist.ac.kr}

\author{Seongwoo Lim}
\orcid{0009-0003-0982-0216}
\affINEEJI
\email{seongwoo.lim@ineeji.com}

\author{Kyowoon Lee}
\authornote{Co-corresponding authors.}
\affKAIST
\email{leekwoon@kaist.ac.kr}

\author{Jaesik Choi}
\authornotemark[1]

\orcid{0000-0002-4663-3263}
\affKAIST
\affINEEJI
\email{jaesik.choi@kaist.ac.kr}

\renewcommand{\shortauthors}{Kim et al.}

\begin{abstract}
Integrated Gradients (IG) is a widely adopted feature attribution method that satisfies desirable axiomatic properties. However, the choice of integration path significantly affects the quality of attributions, and the standard straight-line path introduces all input features simultaneously, often accumulating noisy gradients along the way. To address this limitation, we propose \emph{Spectral Integrated Gradients}, which constructs integration paths based on singular value decomposition (SVD) of the baseline-to-input difference. By progressively activating singular components from largest to smallest, SIG introduces global structure before fine-grained details, naturally following a coarse-to-fine progression. Through extensive evaluation across diverse image classification datasets, we demonstrate that SIG produces cleaner attribution maps with reduced noise and achieves improved quantitative performance compared to existing path-based attribution methods. Our code is available at \sigcode

\end{abstract}



\received{8 February 2026}
\received[revised]{17 April 2026}
\received[accepted]{16 May 2026}

\maketitle

\begin{figure*}[ht]
    \centering
    \includegraphics[width=\textwidth]{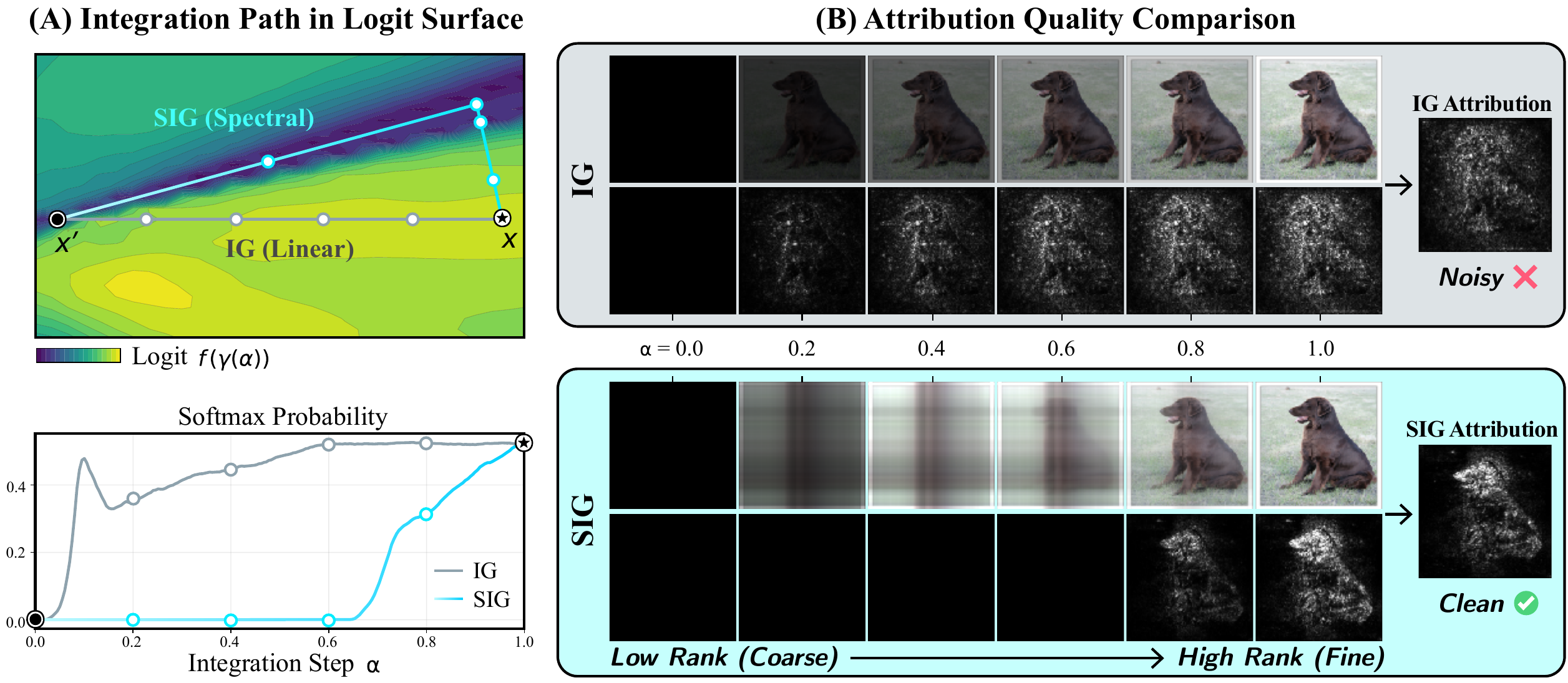}
    \caption{\textbf{Overview of Spectral Integrated Gradients (SIG).}
    \textbf{(A)} In the logit surface landscape, the linear path (gray line) of IG interpolates features uniformly, often traversing unstable regions (see Softmax Probability spikes). In contrast, SIG constructs a path (cyan line) that prioritizes coarse structural changes (low-rank) before introducing fine details (high-rank).
    \textbf{(B)} By adhering to this spectral hierarchy, SIG avoids the indiscriminate accumulation of high-frequency noise visible in IG's early steps. This effectively suppresses spurious gradients, resulting in cleaner attributions.
    Example from the ImageNet validation set classified as `Flat-Coated Retriever' (confidence 0.52) using InceptionV1.
    The integration path proceeds from baseline \protect{\raisebox{-.05cm}{\includegraphics[height=.35cm]{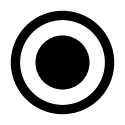}}} to input \protect{\raisebox{-.05cm}{\includegraphics[height=.35cm]{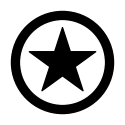}}}.}
    \Description{
    The figure illustrates the concept of Spectral Integrated Gradients (SIG) compared to standard Integrated Gradients (IG) in two panels.
    Panel (A) shows the integration path in a logit surface contour plot. The standard IG path is a straight gray line connecting the baseline x' to the input x, traversing through regions with fluctuating gradients. Below this, a line graph shows the Softmax Probability for IG fluctuating wildly across the integration steps. In contrast, the SIG path is a cyan line that avoids these noisy regions. The corresponding SIG probability graph remains low and stable initially, then rises smoothly towards the end.
    Panel (B) compares the attribution quality using an image of a dog. The top row (IG) displays images fading in linearly from black to full intensity. The resulting IG attribution map is scattered and labeled "Noisy" with a red cross. The bottom row (SIG) displays images evolving from blurry, low-rank structures to sharp, high-rank details. The resulting SIG attribution map is focused on the dog's head and body, labeled "Clean" with a green checkmark, demonstrating that prioritizing coarse structures yields better explanations.
    }
    \label{fig:sig_overview}
\end{figure*}

\section{Introduction}

Deep neural networks have achieved remarkable performance across numerous domains, yet their black-box nature poses significant challenges for deployment in high-stakes applications such as medical diagnosis \citep{chen2021deep} and autonomous systems \citep{lee2023adaptive,lee2025local,lee2025state,lee2026refining}. This inherent opacity makes it difficult to trust, debug, or ensure the safety and fairness of model decisions. To address this challenge, feature attribution methods have emerged as essential tools in explainable artificial intelligence by assigning importance scores to each input feature based on its contribution to the model output \citep{simonyan2014saliency, ribeiro2016why, lundberg2017unified, montavon2019layer, jeon2022distilled, kasmi2025one}. These methods enable practitioners to identify which input regions drive a prediction, facilitating systematic model diagnostics and error analysis~\citep{lapuschkin2019unmasking, adebayo2020debugging, anders2022finding, lee2023refining, kim2023explainable}. However, the reliability of many early attribution methods has been questioned, 
as they often rely on heuristics and have been shown to fail fundamental sanity checks \citep{adebayo2018sanity}, 
raising concerns about their potential to produce misleading explanations \citep{kindermans2019reliability, shah2021input}.

Integrated Gradients (IG) \citep{sundararajan2017axiomatic} emerged as a principled alternative grounded in the Aumann-Shapley value framework and path-based formulations \citep{aumann1974values, friedman1997paths}. Unlike heuristic approaches, IG satisfies several desirable axioms including sensitivity, implementation invariance, and completeness, providing theoretical guarantees leading to its widespread adoption. The method computes attributions by integrating gradients along a straight-line path from a neutral baseline (typically a black image) to the target input, accumulating the contribution of each feature throughout this trajectory. Due to this axiomatic rigor, IG has been widely adopted as a standard tool for feature attribution across diverse model architectures.

Despite these theoretical strengths, the practical quality of IG attributions depends critically on the choice of integration path \citep{sturmfels2020visualizing, kapishnikov2021guided}. The standard straight-line path interpolates linearly between baseline and input, treating all pixel differences uniformly throughout the integration. This uniform treatment presents a fundamental limitation: gradients are accumulated equally from all frequency components of the image difference, including high-frequency noise and fine textures that often contribute spurious attributions. As a result, the final attribution maps frequently exhibit scattered importance scores across visually irrelevant regions.

Recent works have sought to address this issue through alternative path constructions. Guided Integrated Gradients \citep{kapishnikov2021guided} greedily selects features with the smallest gradient magnitudes at each step, thereby bypassing regions with high gradient irregularity on the output surface. Manifold-based approaches such as MIG \citep{zaher2024manifold} and EIG \citep{jha2020enhanced} constrain paths to the latent space of generative models, ensuring intermediate points remain on the data distribution. While these methods improve attribution quality, they either require iterative per-step gradient computations or depend on separately trained generative models, introducing additional complexity and computational cost.

We propose a different perspective grounded in spectral analysis. Our key observation is that singular value decomposition (SVD) provides a natural ranking of the baseline-to-input difference: larger singular values capture dominant structural patterns such as object boundaries and major contrast regions, while smaller singular values encode fine-grained textures and noise. By constructing the integration path to activate large singular components before small ones, we obtain a principled ordering that prioritizes structurally important information (see Figure~\ref{fig:sig_overview} for an overview). This spectral path effectively reduces noisy gradient contributions from less significant components during the early stages of integration, yielding cleaner attribution maps. From this perspective, we introduce \emph{Spectral Integrated Gradients}, a simple yet effective method that requires only a single SVD computation per image. The method decomposes the difference between the input and the baseline into orthogonal rank-one components and progressively scales their contributions along the integration path. While the SVD path generalizes to any matrix-structured input, the coarse-to-fine bias is most beneficial when SVD's variance ordering aligns with spatially coherent semantics, as in natural images; we therefore focus our main evaluation on image classification, following the convention of prior path-based attribution work~\citep{kapishnikov2021guided, xu2020attribution, zhang2024path, pan2021explaining}.

To summarize, our contribution in this paper is the introduction of Spectral Integrated Gradients, a simple and efficient path-based attribution method that leverages singular value decomposition to construct integration paths prioritizing structurally dominant components. Extensive experiments across multiple image classification benchmarks demonstrate that Spectral IG consistently produces cleaner attribution maps with reduced noise compared to existing path-based methods. Furthermore, our method achieves improved quantitative performance while adding negligible computational overhead beyond standard Integrated Gradients.

\section{Background}
\label{sec:background}

In this section, we review the theoretical foundations of path-based attribution and introduce the spectral properties of image data that motivate our approach.

\subsection{Path-based Attribution Methods}

We consider a deep neural network classifier $f: \mathbb{R}^{H \times W} \to \mathbb{R}$ and an input image $x \in \mathbb{R}^{H \times W}$. For simplicity, we treat each color channel independently and present the formulation for a single channel; the extension to multi-channel images is straightforward. Feature attribution methods seek to assign an importance score $\mathcal{A}_i \in \mathbb{R}$ to each input feature $i$, representing its contribution to the model prediction $f(x)$ relative to a baseline $x' \in \mathbb{R}^{H \times W}$ (typically a zero matrix). For notational convenience, we flatten the spatial dimensions and write $x \in \mathbb{R}^n$ where $n = H \times W$, so that each feature corresponds to a single index $i \in \{1, \dots, n\}$.

Path-based attribution methods generalize the attribution problem by accumulating gradients along a continuous curve $\gamma: [0, 1] \to \mathbb{R}^n$ connecting the baseline to the input, such that $\gamma(0) = x'$ and $\gamma(1) = x$. The attribution for feature $i$ is given by the path integral:
\begin{equation}
\label{eq:path_integral}
\mathcal{A}_i(\gamma) = \int_{0}^{1} \frac{\partial f(\gamma(\alpha))}{\partial \gamma_i(\alpha)} \frac{\partial \gamma_i(\alpha)}{\partial \alpha} \, \mathrm{d}\alpha.
\end{equation}
This formulation satisfies several desirable axiomatic properties, such as \textit{Completeness}, which guarantees that the sum of attributions equals the difference in model output: $\sum_i \mathcal{A}_i(\gamma) = f(x) - f(x')$.

Standard Integrated Gradients (IG) \citep{sundararajan2017axiomatic} employs the simplest possible path: the straight line $\gamma(\alpha) = x' + \alpha(x - x')$. While axiomatically sound, the straight-line path assumes that all features evolve uniformly from the baseline to the input. This assumption is often suboptimal for high-dimensional image data, where the \emph{signal} (structural content) and \emph{noise} (high-frequency artifacts) may have distinct effects on the model's gradients at different scales.

\subsection{Spectral Decomposition of Visual Differences}

Digital images possess a natural hierarchy of information, often characterized by their spectral components. A common technique to analyze this hierarchy is Singular Value Decomposition (SVD). 

Let $\Delta = x - x'$ be the difference tensor between the input and the baseline. For simplicity, consider $\Delta$ as a 2D matrix of dimensions $H \times W$ (treating color channels independently or flattened). SVD decomposes this difference into a sum of rank-one matrices:
\begin{equation}
    \Delta = U \Sigma V^\top = \sum_{i=1}^{R} \sigma_i u_i v_i^\top,
\end{equation}
where $R = \min(H, W)$ is the rank, $\sigma_i$ are the singular values ordered $\sigma_1 \ge \sigma_2 \ge \dots \ge 0$, and $u_i, v_i$ are the corresponding left and right singular vectors.

This decomposition provides a principled ordering of visual information. The leading components (associated with large $\sigma_i$) capture the dominant energy, typically corresponding to global structures, shapes, and low-frequency patterns. Conversely, the trailing components (small $\sigma_i$) capture fine-grained details, high-frequency textures, and often, noise.

\section{Spectral Integrated Gradients}
\label{sec:method}

In this section, we introduce \textbf{S}pectral \textbf{I}ntegrated \textbf{G}radients (\textbf{SIG}), a framework that constructs attribution paths by respecting the spectral hierarchy of visual information. We first identify the limitations of linear interpolation through the lens of signal processing, then formulate our path construction as a constrained optimization problem, and finally present a continuous relaxation that yields a practical and effective attribution algorithm.

\subsection{Signal-to-Noise Problem in Linear Paths}
\label{subsec:signal_to_noise}

The limitations of the linear path in IG can be analyzed through this spectral lens. The Integrated Gradients method constructs a path $\gamma_{\text{IG}}(\alpha) = x' + \alpha(x - x')$ that interpolates all features uniformly. Let $\Delta = x - x'$ denote the total difference signal. The linear path scales the entire difference vector $\Delta$ by a scalar $\alpha$:
\begin{equation}
    \gamma_{\text{IG}}(\alpha) = x' + \alpha \sum_{i=1}^{R} \sigma_i u_i v_i^\top.
\end{equation}
This implies that at any step $\alpha$, the path introduces the dominant structural component ($\sigma_1 u_1 v_1^\top$) and the fine-scale detail ($\sigma_R u_R v_R^\top$) at the exact same rate.

This uniform scaling is problematic for gradient integration. Deep networks are famously sensitive to high-frequency perturbations \citep{xu2020attribution, geirhos2020shortcut}. When the integration path introduces high-frequency noise components early in the trajectory, it can trigger erratic gradients that do not reflect the semantic content of the image. These noisy gradients are then accumulated into the final attribution map, resulting in the characteristic speckled noise pattern often observed in standard IG explanations. Consequently, we propose a \textit{coarse-to-fine} integration strategy defined by a path that prioritizes the accumulation of gradients associated with robust semantic structures before introducing fine-grained details. We note that SVD is, strictly speaking, a variance/energy decomposition rather than a frequency decomposition; the coarse-to-fine interpretation is an empirical correlation that holds well for natural images (Section~\ref{sec:frequency}), and \emph{singular-component ordering} is the more precise description of what SIG controls.

\subsection{Path Construction via Spectral Optimization}

To formalize the coarse-to-fine principle, we frame path construction as an optimization problem. We aim to ensure that at any point $\alpha$ along the path, the intermediate difference $\Delta(\alpha) = \gamma(\alpha) - x'$ preserves the maximum amount of information from the target difference $\Delta$, subject to a \emph{complexity budget}.

We define complexity using the rank of the matrix representation of the difference. Let $\Delta \in \mathbb{R}^{H \times W}$ (processed per-channel) have rank $R$. We define the optimal intermediate difference $\Delta^*(\alpha)$ at step $\alpha$ as the solution to:
\begin{equation}
\label{eq:optimization}
    \Delta^*(\alpha) \in \operatorname*{argmin}_{\hat{\Delta}} \| \Delta - \hat{\Delta} \|^2 \quad \text{s.t.} \quad \operatorname{rank}(\hat{\Delta}) \le k(\alpha),
\end{equation}
where $k(\alpha)$ is a monotonically increasing rank budget function. This objective minimizes the squared reconstruction error under a strict complexity constraint. By the Eckart-Young-Mirsky theorem \cite{eckart1936approximation}, the optimal solution to Eq.~\eqref{eq:optimization} is the truncated Singular Value Decomposition (SVD) of $\Delta$. If $\Delta = \sum_{i=1}^R \sigma_i u_i v_i^\top$, then:
\begin{equation}
    \Delta^*(\alpha) = \sum_{i=1}^{k(\alpha)} \sigma_i u_i v_i^\top.
\end{equation}

This result provides a theoretical justification for using SVD: the singular components provide the optimal hierarchy for progressively revealing the image difference under a least-squares distortion metric. This defines a spectral path that activates components sequentially from the largest singular value (coarse structure) to the smallest (fine detail), as illustrated in Figure~\ref{fig:sig_overview}. In contrast, GIG uses greedy coordinate selection, BIG uses Gaussian blur, and SAMP uses stochastic sampling; to our knowledge, these alternatives do not provide an analogous rank-constrained optimality guarantee for path construction.

\paragraph{\textbf{Optimality vs.\ faithfulness.}} We emphasize that the rank-constrained optimality in Eq.~\eqref{eq:optimization} justifies the \emph{path construction criterion}, not the causal or semantic faithfulness of the resulting attributions. Whether a particular path yields more faithful attributions is an empirical question that we address through multiple complementary evaluations (perturbation-based, retraining-based, leakage-aware, and non-perturbation localization) in Section~\ref{sec:quant} and Appendix~\ref{app:road}--\ref{app:localization}. Note that this caveat is not unique to SIG: IG's straight-line path is likewise not derived from a faithfulness guarantee, and its standard axioms constrain the attribution rule rather than the path.

\subsection{Continuous Relaxation and Overlap}

\def\NoNumber#1{\STATE \textcolor{gray}{#1}}

\begin{figure}[t]
\vspace{-0.3cm}
\begin{minipage}{\linewidth}
\begin{algorithm}[H]
    \caption{\textbf{Spectral Integrated Gradients (SIG)}}
    \label{alg:sig}
    \begin{algorithmic}
    \STATE \textbf{Input:} Image $x$, baseline $x'$, model $f$, steps $M$, overlap parameter $\omega$
    \STATE \textbf{Output:} Attribution map $\mathcal{A}$
    
    \NoNumber{\small{\color{gray}\texttt{// spectral decomposition}}}
    \STATE $\Delta \gets x - x'$
    \STATE $U, \bm{\sigma}, V^\top \gets \operatorname{SVD}(\Delta)$ \hfill \COMMENT{$\bm{\sigma} = (\sigma_1, \dots, \sigma_R)$}
    
    \NoNumber{\small{\color{gray}\texttt{// path generation}}}
    \STATE $x^{(0)} \gets x'$
    
    \FOR{$m = 0 \text{ to } M-1$}
        \STATE $\alpha \gets (m+1) / M$
        
        \NoNumber{\small{\color{gray}\texttt{// compute spectral scaling factors (Eq.\ref{eq:scaling_fn})}}}
        \STATE $\tilde{\bm{\sigma}} \gets \bm{\sigma} \odot \bm{\phi}(\alpha; \omega)$ \hfill \COMMENT{$\tilde{\sigma}_i = \sigma_i \cdot \phi_i(\alpha; \omega)$}
        
        \NoNumber{\small{\color{gray}\texttt{// reconstruct partial difference and next point}}}
        \STATE $\Delta^{(m+1)} \gets U \cdot \operatorname{diag}(\tilde{\bm{\sigma}}) \cdot V^\top$
        \STATE $x^{(m+1)} \gets x' + \Delta^{(m+1)}$
    \ENDFOR
    
    \NoNumber{\small{\color{gray}\texttt{// compute path integral attributions}}}
    \STATE $\mathcal{A} \gets \sum_{m=0}^{M-1} \nabla_x f(x^{(m)}) \odot (x^{(m+1)} - x^{(m)})$
    
    \STATE \textbf{return} $\mathcal{A}$
    \end{algorithmic}
\end{algorithm}
\end{minipage}
\Description{Pseudocode for Spectral Integrated Gradients (SIG).}
\end{figure}

While the truncated SVD path is theoretically optimal for a discrete rank budget, it introduces discontinuities where singular components are abruptly activated. Such discontinuities are undesirable for numerical integration. To address this, we propose a \emph{continuous relaxation} of the hard rank constraint.

We replace the binary activation of singular values with a smooth gating function. Let $\omega \in (0, 1]$ be an \emph{overlap parameter} that controls the \emph{softness} of the transition. We define a continuous scaling function $\phi_i(\alpha; \omega)$ for the $i$-th singular component at path progress $\alpha \in [0, 1]$:
\begin{equation}
\label{eq:scaling_fn}
    \phi_i(\alpha; \omega) = \min\left(\max\left( \frac{\alpha - s_i}{e_i - s_i}, 0\right), 1 \right),
\end{equation}
where $s_i$ and $e_i$ (satisfying $0 \leq s_i < e_i \leq 1$) denote the \emph{activation start} and \emph{full activation} thresholds for component $i$, respectively. These are determined by the relative spectral importance:
\begin{equation}
    s_i = (1-\omega)\left(1 - \frac{\sigma_i}{\sigma_1}\right), \quad e_i = s_i + \omega.
\end{equation}
Intuitively, $s_i$ is the path progress at which component $i$ begins to contribute, and $e_i$ is the threshold at which it reaches full contribution ($\phi_i = 1$). Larger singular values (closer to $\sigma_1$) have smaller $s_i$, causing them to activate earlier in the path. The overlap parameter $\omega$ controls how much the activation windows of different components overlap: larger $\omega$ leads to more simultaneous activation, while smaller $\omega$ enforces stricter sequential ordering.
The Spectral Integrated Gradients path is then defined as:
\begin{equation}
    \gamma_{\text{SIG}}(\alpha) = x' + \sum_{i=1}^{R} \phi_i(\alpha; \omega) \sigma_i u_i v_i^\top.
\end{equation}

This formulation positions $\omega$ as a mechanism to interpolate between distinct path behaviors. In the limit as $\omega \to 0$, SIG approximates the truncated SVD path, strictly adhering to a sequential coarse-to-fine spectral order. Conversely, as $\omega \to 1$, the start times $s_i$ approach zero for all components, causing them to scale linearly with $\alpha$. Consequently, in this limit, SIG recovers the standard linear IG path. In practice, we use $\omega = 0.4$ to ensure path continuity while preserving spectral ordering. The complete procedure for Spectral Integrated Gradients is summarized in Algorithm~\ref{alg:sig}.

\subsection{Axiomatic Properties of Spectral IG}
\label{subsec:axioms}

As a path-based attribution method, Spectral Integrated Gradients (SIG) is grounded in the axiomatic framework established by \cite{sundararajan2017axiomatic}. By defining the attribution as a path integral along a continuous trajectory $\gamma_{\text{SIG}}$ that satisfies the boundary conditions $\gamma(0)=x'$ and $\gamma(1)=x$, SIG inherently satisfies several fundamental axioms.

First, \emph{Completeness}, which ensures that the sum of attributions equals $f(x) - f(x')$, is guaranteed by the fundamental theorem of calculus for line integrals \cite{sundararajan2017axiomatic}. Second, \emph{Sensitivity-b} is satisfied because the partial derivative with respect to any non-functional feature is identically zero everywhere along the path, causing the attribution integral to vanish. Third, \emph{Implementation Invariance} holds because the integration path is derived solely from the SVD of the input difference $\Delta$. Since gradients are invariant for functionally equivalent models, the resulting path integrals are identical regardless of the specific network architecture. Inheriting these properties from the IG framework is a deliberate design choice: by changing only the path, SIG preserves IG's full axiomatic guarantees while improving practical attribution quality. Finally, we formally establish that SIG preserves the \emph{symmetry} of input features. This property is not automatically guaranteed by adaptive path methods such as AGI or IG$^2$, but is preserved here due to the geometric properties of the Singular Value Decomposition.

\begin{proposition}[Geometric Symmetry Preservation]
\label{prop:geometry}
Let $T$ be a permutation matrix representing a discrete spatial operation (e.g., $90^\circ$-multiple rotation or reflection).
Let $f:\mathbb{R}^n\to\mathbb{R}$ satisfy $f(Tx)=f(x)$ for all $x$.
Then SIG is equivariant under $T$:
\[
\mathcal{A}_{SIG}(Tx',Tx)=T\,\mathcal{A}_{SIG}(x',x).
\]
\end{proposition}

\begin{sproof}
The proof follows from the geometric properties of the SVD and the Hadamard product. First, since $T$ is a permutation matrix (hence orthogonal), the SVD of the transformed difference $\Delta_T = T\Delta$ preserves the singular values $\Sigma$ while spatially permuting the singular vectors. This implies the SIG path transforms equivariantly: $\gamma_T(\alpha) = T\gamma(\alpha)$.
Second, the symmetry of $f$ implies gradient equivariance: $\nabla f(Tx) = T\nabla f(x)$.
Finally, a key property of permutation matrices is that they distribute over the element-wise product: $(Ta)\odot(Tb) = T(a\odot b)$. Combining these, the transformation $T$ factors out of the attribution integral.
We refer the reader to Appendix~\ref{appendix:proofs} for the detailed formal derivation.
\end{sproof}

Geometric Symmetry Preservation is not automatically guaranteed by adaptive path methods, since a data-dependent path construction can break the equivariance between transformed inputs and their attributions. In SIG, because singular values are invariant under orthogonal transformations, the spectral scaling schedule remains unchanged when the input is transformed, and the transformation factors out of the attribution integral. In practice, this means that when a classifier is invariant to a spatial transformation, the attribution map transforms accordingly, ensuring spatially consistent explanations.



\section{Experiments}
\label{sec:experiments}

In this section, we empirically validate the effectiveness of Spectral Integrated Gradients. We investigate \textbf{(1)} whether Spectral IG produces qualitatively more human-perceptually aligned attributions, \textbf{(2)} whether Spectral IG achieves improved quantitative performance compared to existing attribution methods, and \textbf{(3)} the sensitivity of the method to hyperparameter choices.

\subsection{Experimental Setup}

\paragraph{\textbf{Datasets and Models.}}
We conduct experiments on three image classification benchmarks with varying characteristics: (1) \textbf{ImageNet}~\citep{deng2009imagenet}, the standard large-scale benchmark with 1,000 object categories; (2) \textbf{Oxford-IIIT Pet}~\citep{parkhi2012cats}, a fine-grained dataset containing 37 pet breeds where subtle visual features distinguish similar classes; and (3) \textbf{Oxford 102 Flower}~\citep{nilsback2008automated}, comprising 102 flower species with diverse scales, poses, and backgrounds.
For classifiers, we use VGG16~\citep{simonyan2015very}, ResNet18~\citep{he2016deep}, and InceptionV1~\citep{szegedy2015going}, all initialized with ImageNet-pretrained weights. Models for the Oxford datasets are fine-tuned on their respective training splits.

\paragraph{\textbf{Baselines.}}

We compare SIG against a comprehensive set of attribution methods. We include G$\times$I~\citep{shrikumar2016not} as a gradient-based baseline. For path-based approaches, we evaluate the foundational IG~\citep{sundararajan2017axiomatic} and distinct path construction strategies: adaptive methods including GIG~\citep{kapishnikov2021guided} which greedily avoids high-curvature regions, IG$^2$~\citep{zhuo2024ig2} which seeks steeper gradients to avoid saturation, AGI~\citep{pan2021explaining} which integrates along steepest-ascent paths from adversarial examples, BIG~\citep{xu2020attribution} which integrates along a Gaussian deblurring path satisfying scale-space axioms, and SAMP~\citep{zhang2024path} which optimizes paths under a concentration principle via stochastic search; and manifold-based approaches EIG~\citep{jha2020enhanced} and MIG~\citep{zaher2024manifold}, which employ linear interpolation and geodesic paths within the VAE latent space, respectively.

\paragraph{\textbf{Evaluation Metrics.}}

We employ two complementary metrics to evaluate attribution faithfulness. Difference between Insertion and Deletion Games (\textbf{DiffID})~\citep{yang2022re, yang2023local} builds on the Insertion and Deletion protocols~\citep{petsiuk2018rise}, which track how model confidence changes as pixels are progressively added or removed according to their attributed importance. Rather than evaluating each curve independently, DiffID takes their difference, canceling out distributional biases that arise from modifying input images. 
RemOve And Retrain (\textbf{ROAR})~\citep{hooker2019benchmark} provides a global evaluation by retraining the model on data where top-$k$\% attributed pixels are replaced with mean values, then retraining the model from random initialization on the modified data. This retraining step is essential to decouple the effect of distribution shift from genuine information removal, ensuring the model learns to treat replacement values as uninformative. Lower accuracy after retraining indicates more accurate identification of discriminative features.

\begin{figure}[!t]
    \centering
    \includegraphics[width=\linewidth]{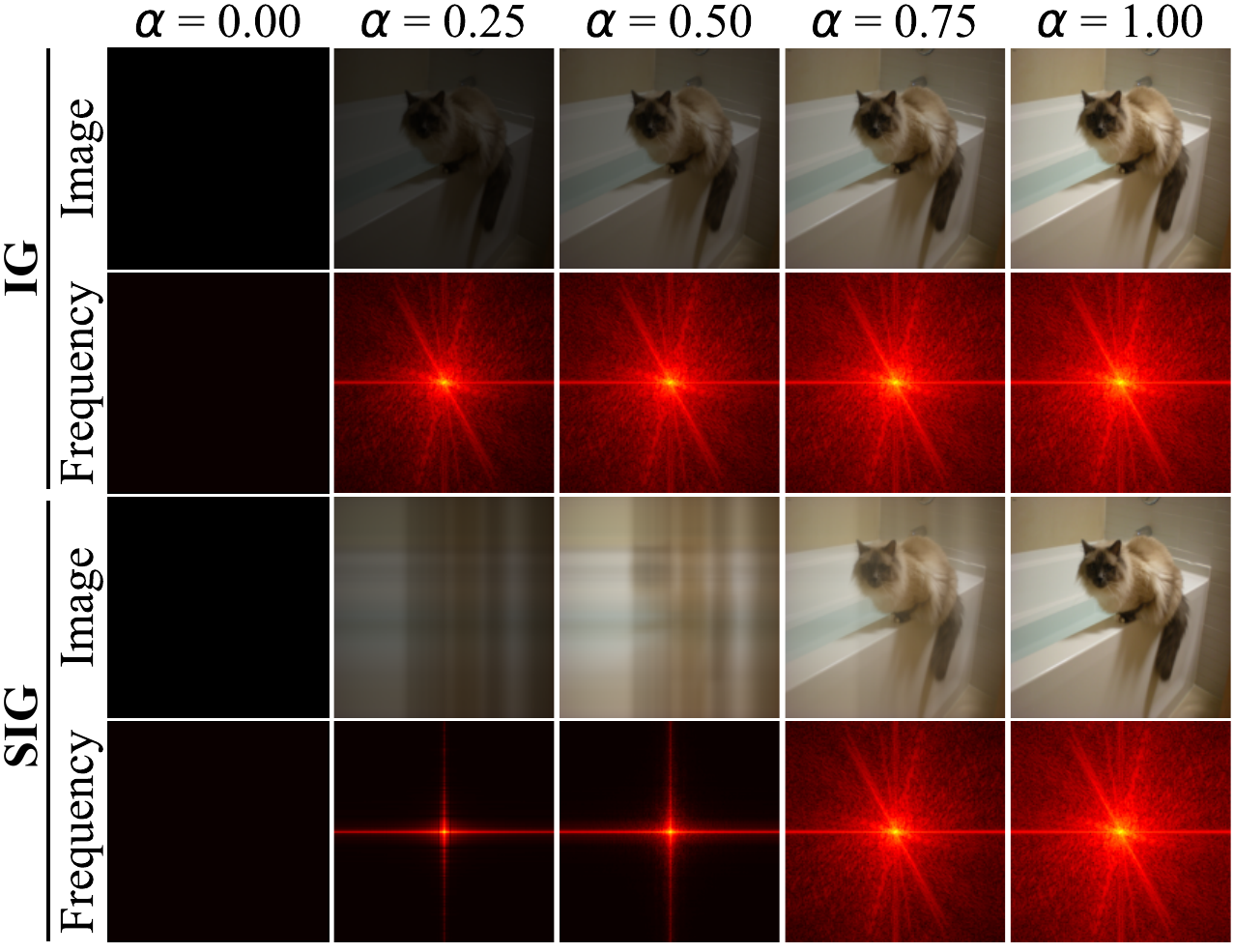}
    \caption{\textbf{Frequency-domain analysis of integration paths.} Rows 1--2 show Standard IG (linear interpolation), and rows 3--4 show SIG (spectral interpolation). For each method, we visualize the interpolated image and the 2D Fourier magnitude of the baseline difference $\Delta(\alpha)=\gamma(\alpha)-x'$. IG introduces frequency components uniformly, whereas SIG follows a coarse-to-fine trajectory: low-frequency structure appears first, followed by high-frequency detail and noise.}
    \label{fig:frequency}
    \Description{The figure displays four rows of images sampled at five integration steps (alpha equals 0, 0.25, 0.5, 0.75, and 1). The top two rows correspond to standard IG: the first row shows interpolated images fading in uniformly from black to full intensity, and the second row shows the 2D Fourier magnitude spectrum of the difference from baseline, where all frequency components appear simultaneously at every step. The bottom two rows correspond to SIG: the first row shows images evolving from blurry low frequency structures to sharp details, and the second row shows the Fourier spectrum gradually expanding from a concentrated low frequency center outward to high frequency components as alpha increases.}
\end{figure}

\paragraph{\textbf{Implementation Details.}}
For a comprehensive description of hyperparameters and experimental configurations, please refer to Appendix~\ref{app:setup}. 
We employed VGG16, ResNet18, and InceptionV1 classifiers initialized with ImageNet-pretrained weights from the Torchvision library~\citep{torchvision2016}. For the Oxford-IIIT Pet and Oxford 102 Flower benchmarks, the models were fine-tuned on their respective training splits. 
To ensure fair comparison, we standardized the number of integration steps to $M = 200$ for most path-based methods. An exception is AGI, for which we employed 15 iterations following the official implementation settings~\citep{pan2021explaining}. 
For SIG, the overlap parameter was set to $\omega = 0.4$. 
Regarding the baseline, we followed standard practice~\citep{sundararajan2017axiomatic} by using a black image (zero tensor) for IG, GIG, EIG, MIG, and SIG. 
However, methods with distinct path formulations deviate from this setting: BIG integrates from a blurred baseline, AGI utilizes adversarial starting points, and $IG^2$ employs a random reference mode. 
The implementation code for SIG is provided in Appendix~\ref{app:code}.

\subsection{Frequency-Domain Analysis}

\begin{figure}[!t]
    \centering
    \includegraphics[width=\linewidth]{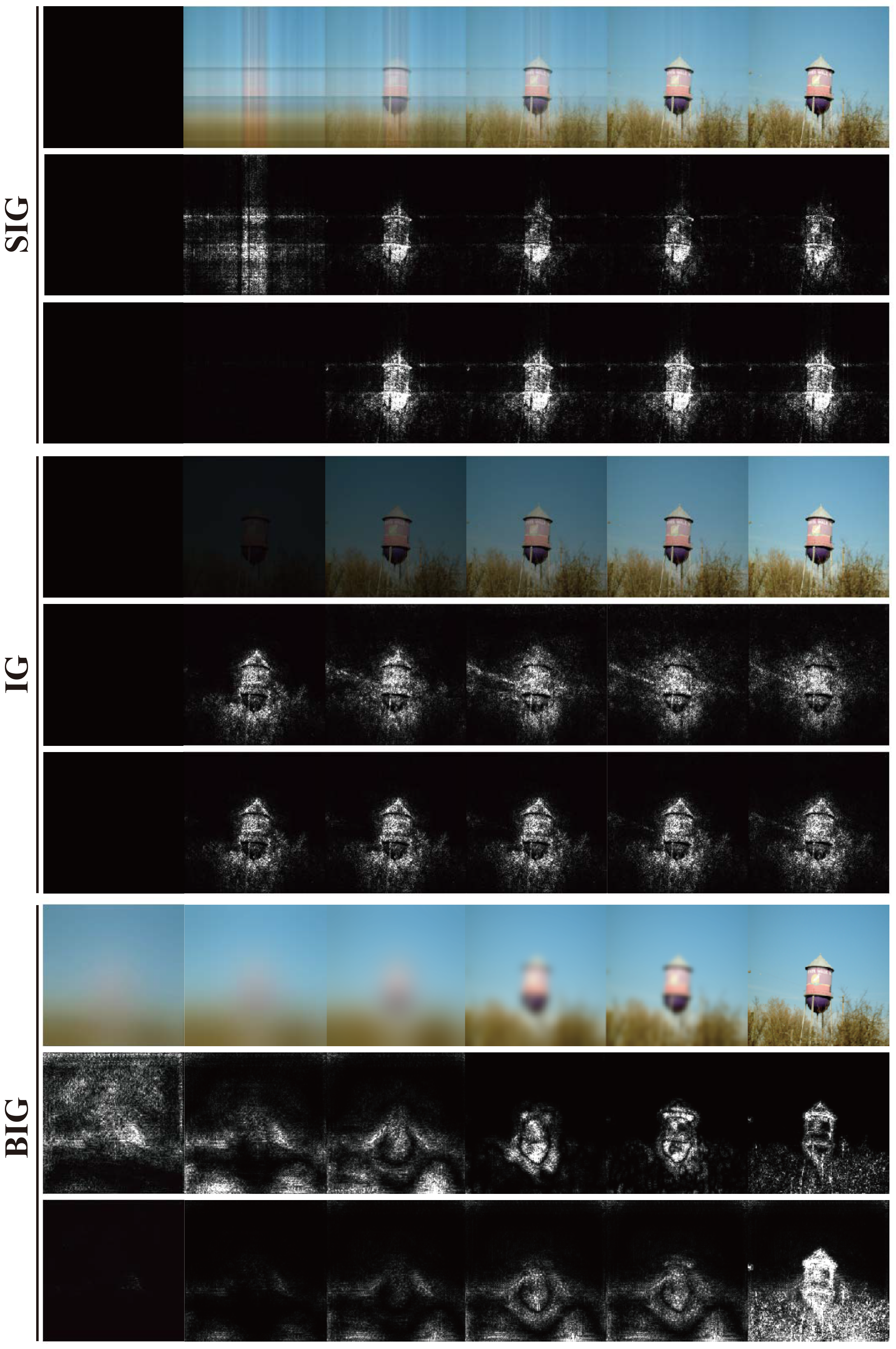}
    \caption{Path analysis on ImageNet using the VGG16 classifier. The predicted class is `water tower' with a confidence of 0.963. The first, second, and third rows represent the path features $x^{(m)}$, the gradients multiplied by the path difference $|\nabla_x f(x^{(m)}) \odot (x^{(m+1)} - x^{(m)})|$, and the aggregated attribution $|\mathcal{A}^{(m)}|$ at step $m$.}
    \label{fig:path_vgg}
    \Description{The figure compares three attribution methods (IG, BIG, and SIG) on an ImageNet image classified as water tower. Each block consists of three rows: interpolated path images, per step gradient contributions, and cumulative attribution maps, displayed across 10 evenly spaced integration steps. In the IG block, gradient contributions appear noisy and scattered throughout all steps, producing a diffuse final attribution map. In the BIG block, gradients concentrate along object boundaries, resulting in an edge dominated attribution. In the SIG block, gradient contributions remain suppressed during early steps and emerge only after coarse structure is established, yielding an attribution map that is focused on the semantically relevant regions of the water tower.}
\end{figure}

\begin{figure*}[t]
    \centering
    \includegraphics[width=\linewidth]{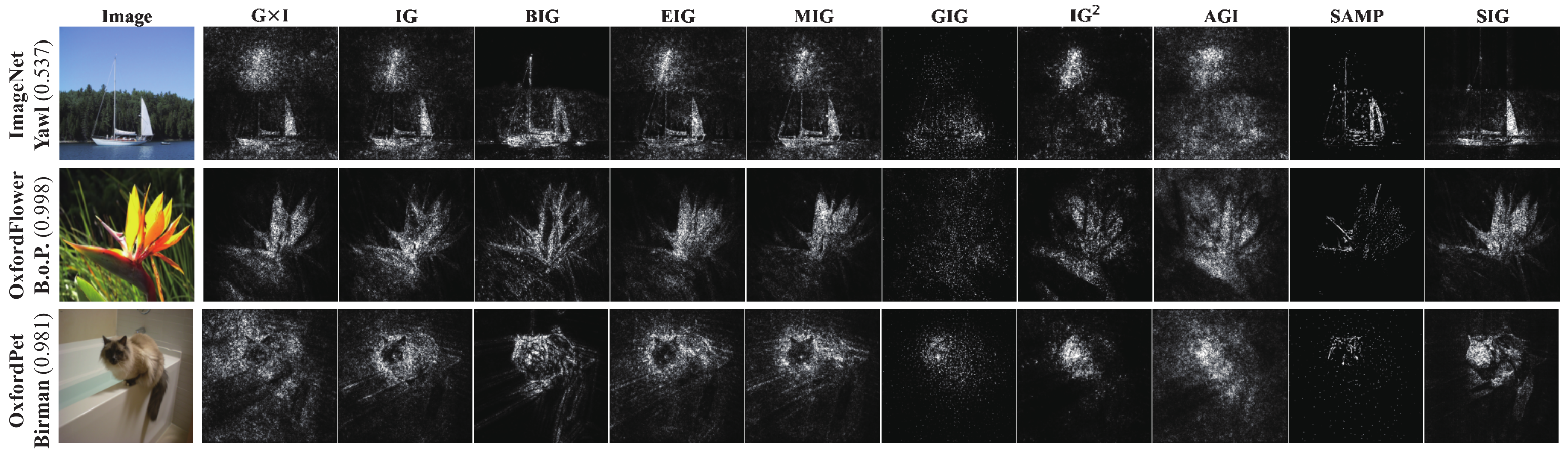}
    \caption{\textbf{Qualitative comparison of attribution maps on ImageNet (InceptionV1), Oxford-IIIT Pet (ResNet18), and Oxford 102 Flower (VGG16) against baselines.} Labels indicate predicted classes, and numbers in brackets denote prediction confidence. (B.o.P.: Bird of Paradise)}
    \Description{
    A figure displaying a grid of qualitative attribution results with three rows and eleven columns. 
    The rows represent three different input samples: a 'Yawl' from ImageNet, a 'Bird of Paradise' flower from OxfordFlower, and a 'Birman' cat from OxfordPet. 
    The columns compare the original input image against ten different attribution methods: GxI, IG, BIG, EIG, MIG, GIG, IG-squared, AGI, SAMP, and the proposed SIG.
    Visually, baseline methods like IG, GxI, and BIG exhibit significant background noise; for instance, in the boat image, they highlight the sky and water with scattered white pixels. 
    The SAMP method produces attributions that appear sharp but disjointed, resembling 'salt-and-pepper' noise or isolated speckles rather than continuous features. 
    In contrast, the proposed SIG method in the final column produces the cleanest and sparsest attribution maps. 
    SIG suppresses background clutter (such as the wall behind the cat or the sky above the boat) and focuses intensity on the core structural elements of the objects, such as the boat's mast, the flower's petals, and the cat's face.
    }
    \label{fig:qual}
\end{figure*}

\begin{figure}[ht]
    \centering
    \includegraphics[width=\linewidth]{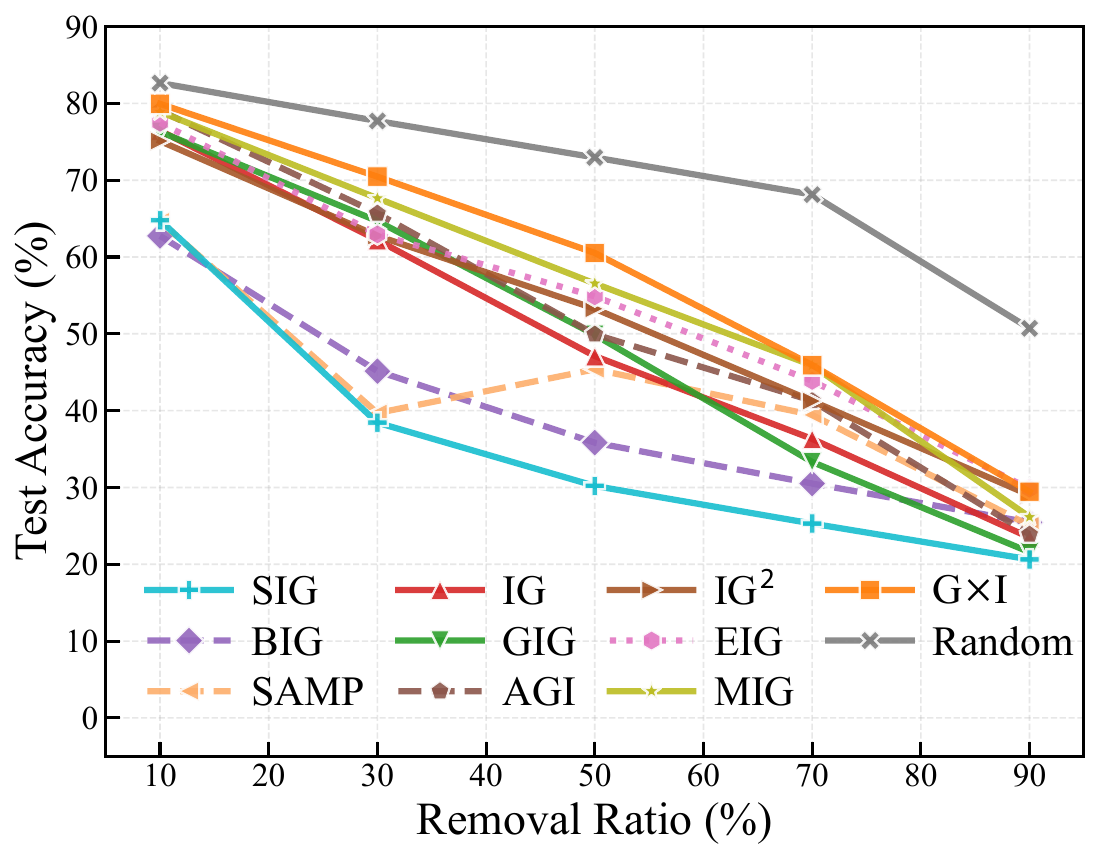}
    \caption{\textbf{RemOve And Retrain (ROAR) test on CIFAR10.} SIG (cyan line) achieves the lowest area under the curve (AUC) among all compared methods.}
    \Description{A line chart plotting test accuracy (y-axis) against the percentage of top-attributed pixels removed (x-axis) for the ROAR evaluation on CIFAR-10. Each line represents a different attribution method. As more pixels identified as important are removed and the model is retrained, accuracy decreases for all methods. SIG (cyan) exhibits the steepest accuracy drop and the lowest overall area under the curve, indicating that SIG most accurately identifies discriminative features. Other methods such as IG, GIG, and BIG show comparatively slower accuracy degradation, suggesting less precise feature identification.}
    \label{fig:roar}
\end{figure}

\label{sec:frequency}

To examine the spectral behavior of SIG, we analyze how frequency content evolves along the integration path. At each step $\alpha \in \{0, 0.25, 0.5, 0.75, 1\}$, we compute the 2D Discrete Fourier Transform (DFT) of the difference from baseline $\Delta(\alpha) = \gamma(\alpha) - x'$ and visualize the log-scaled magnitude spectrum. 
As illustrated in \Cref{fig:frequency}, although SVD is strictly a variance-based decomposition rather than a frequency-based one, our analysis reveals an empirical correlation: the SIG path exhibits a distinct coarse-to-fine spectral progression. Dominant low-frequency structures emerge during early steps of $\alpha$, while high-frequency details, which often encompass noise, are progressively introduced only at later stages. This ordering enables SIG to prioritize robust structural features for attribution while effectively suppressing the accumulation of high-frequency noise (see Appendix~\ref{app:freq} for additional examples).

\subsection{Qualitative Comparison}

We qualitatively compare the attribution maps generated by SIG against various baseline methods in \Cref{fig:qual}. Overall, SIG consistently produces attribution maps that are visibly less noisy and more focused on feature areas relevant to the predicted class. While standard gradient-based methods often assign significant importance to irrelevant edges or background textures, SIG effectively suppresses these spurious signals, generating sparse attributions that minimize clutter. Furthermore, unlike methods such as SAMP, which tend to produce speckled attributions with isolated spikes in pixel importance, SIG allocates attribution coherently across the relevant feature area, preserving structural integrity. This superior performance stems directly from our coarse-to-fine integration mechanism. By prioritizing the aggregation of gradients associated with dominant structural components, SIG prevents the accumulation of erratic gradients that arise when fine-scale details are introduced too early along the integration path.

\subsection{Path Analysis}
\label{sec:path}

\begin{table*}[ht]
\centering
\caption{\textbf{Quantitative comparison on Oxford-IIIT Pet, Oxford 102 Flower, and ImageNet datasets using ResNet18, VGG16, and InceptionV3 classifiers.} Best results are highlighted in \textbf{Bold}, and second-best results are \underline{underlined}.}
\label{tab:quant}
\setlength{\tabcolsep}{7.5pt}
\begin{tabular}{ll rrr rrr rrr rrr}
\toprule
& & \multicolumn{3}{c}{\textbf{ResNet18}} & \multicolumn{3}{c}{\textbf{VGG16}} & \multicolumn{3}{c}{\textbf{InceptionV3}} \\
\cmidrule(lr){3-5} \cmidrule(lr){6-8} \cmidrule(lr){9-11} 
\textbf{Data} & \textbf{Method} & \textbf{DiffID ($\uparrow$)} & \textbf{Ins ($\uparrow$)} & \textbf{Del ($\downarrow$)} & \textbf{DiffID ($\uparrow$)} & \textbf{Ins ($\uparrow$)} & \textbf{Del ($\downarrow$)} & \textbf{DiffID ($\uparrow$)} & \textbf{Ins ($\uparrow$)} & \textbf{Del ($\downarrow$)} \\
\midrule
\cellcolor{white} & \textbf{G $\times$ I} \cite{shrikumar2016not}    & 0.2384 & 0.4378 & 0.1994 & 0.4060 & 0.5174 & 0.1114 & 0.2255 & 0.3940 & 0.1685 \\ 
\rowcolor{palegray}
\cellcolor{white} & \textbf{IG} \cite{sundararajan2017axiomatic}     & 0.3790 & 0.5186 & 0.1396 & 0.5258 & 0.6060 & 0.0802 & 0.3438 & 0.4748 & 0.1309 \\ 
\cellcolor{white} & \textbf{BIG} \cite{xu2020attribution}            & \textbf{0.5462} & \underline{0.6228} & \underline{0.0766} & \underline{0.6252} & \underline{0.6910} & \textbf{0.0658} & \underline{0.5246} & \underline{0.5922} & \textbf{0.0676} \\ 
\rowcolor{palegray}
\cellcolor{white} & \textbf{EIG} \cite{jha2020enhanced}              & 0.3417 & 0.4946 & 0.1529 & 0.4853 & 0.5751 & 0.0898 & 0.3228 & 0.4646 & 0.1417 \\ 
\cellcolor{white} & \textbf{MIG} \cite{zaher2024manifold}            & 0.3586 & 0.4982 & 0.1396 & 0.5009 & 0.5829 & 0.0820 & 0.3273 & 0.4652 & 0.1378 \\ 
\rowcolor{palegray}
\cellcolor{white} & \textbf{GIG} \cite{kapishnikov2021guided}        & 0.3634 & 0.5093 & 0.1459 & 0.4709 & 0.5550 & 0.0841 & 0.3586 & 0.4880 & 0.1294 \\ 
\cellcolor{white} & \textbf{AGI} \cite{pan2021explaining}            & 0.2787 & 0.4453 & 0.1667 & 0.4447 & 0.5345 & 0.0898 & 0.3381 & 0.4589 & 0.1207 \\ 
\rowcolor{palegray}
\cellcolor{white} & \textbf{IG$^2$} \cite{zhuo2024ig2}               & 0.3823 & 0.5264 & 0.1441 & 0.6060 & 0.6832 & 0.0772 & 0.4273 & 0.5315 & 0.1042 \\ 
\cellcolor{white} & \textbf{SAMP} \cite{zhang2024path}               & 0.4655 & 0.5315 & \textbf{0.0661} & 0.5357 & 0.6162 & 0.0805 & 0.4751 & 0.5655 & 0.0904 \\ 
\cmidrule(lr){2-11}
\rowcolor{lightgray}
    \multirow{-11}{*}{\rotatebox[origin=c]{90}{\textbf{Oxford-IIIT Pet}}} 
\cellcolor{white} & \textbf{SIG (ours)}                             & \underline{0.5210} & \textbf{0.6237} & 0.1027 & \textbf{0.6532} & \textbf{0.7270} & \underline{0.0739} & \textbf{0.5330} & \textbf{0.6114} & \underline{0.0784} \\ 

\midrule
\cellcolor{white} & \textbf{G $\times$ I} \cite{shrikumar2016not}    & 0.1222 & 0.2338 & 0.1116 & 0.2784 & 0.3576 & 0.0791 & 0.2000 & 0.2813 & 0.0813 \\ 
\rowcolor{palegray}
\cellcolor{white} & \textbf{IG} \cite{sundararajan2017axiomatic}     & 0.1769 & 0.2740 & 0.0971 & 0.3178 & 0.3847 & 0.0669 & 0.2551 & 0.3247 & 0.0696 \\ 
\cellcolor{white} & \textbf{BIG} \cite{xu2020attribution}            & \underline{0.2629} & \underline{0.3131} & \textbf{0.0502} & 0.3116 & 0.3711 & \textbf{0.0596} & 0.2556 & 0.3140 & \underline{0.0584} \\ 
\rowcolor{palegray}
\cellcolor{white} & \textbf{EIG} \cite{jha2020enhanced}              & 0.1871 & 0.2842 & 0.0971 & \underline{0.3218} & \textbf{0.3860} & 0.0642 & 0.2553 & 0.3249 & 0.0696 \\ 
\cellcolor{white} & \textbf{MIG} \cite{zaher2024manifold}            & 0.1738 & 0.2729 & 0.0991 & 0.3076 & 0.3751 & 0.0676 & 0.2531 & 0.3260 & 0.0729 \\ 
\rowcolor{palegray}
\cellcolor{white} & \textbf{GIG} \cite{kapishnikov2021guided}        & 0.1891 & 0.272 & 0.0829 & 0.2587 & 0.3449 & 0.0862 & 0.2611 & \textbf{0.3413} & 0.0802  \\
\cellcolor{white} & \textbf{AGI} \cite{pan2021explaining}            & 0.0136 & 0.1569 & 0.1433 & 0.1409 & 0.2689 & 0.1280 & 0.0787 & 0.1947 & 0.1160 \\ 
\rowcolor{palegray}
\cellcolor{white} & \textbf{IG$^2$} \cite{zhuo2024ig2}               & 0.0193 & 0.1713 & 0.1520 & 0.2189 & 0.3293 & 0.1104 & 0.0816 & 0.2053 & 0.1238 \\ 
\cellcolor{white} & \textbf{SAMP} \cite{zhang2024path}               & \textbf{0.2751} & \textbf{0.3498} & 0.0747 & 0.2869 & 0.3638 & 0.0769 & \underline{0.2662} & \underline{0.3380} & 0.0718 \\ 
\cmidrule(lr){2-11}
\rowcolor{lightgray}
    \multirow{-11}{*}{\rotatebox[origin=c]{90}{\textbf{Oxford 102 Flower}}} 
\cellcolor{white} & \textbf{SIG (ours)}                            & 0.2533 & 0.3116 & \underline{0.0582} & \textbf{0.3224} & \underline{0.3856} & \underline{0.0631} & \textbf{0.2822} & 0.3316 & \textbf{0.0493} \\ 

\midrule
\cellcolor{white} & \textbf{G $\times$ I} \cite{shrikumar2016not}    & 0.1038 & 0.2278 & 0.1240 & 0.1882 & 0.2749 & 0.0867 & 0.1380 & 0.2776 & 0.1396 \\ 
\rowcolor{palegray}
\cellcolor{white} & \textbf{IG} \cite{sundararajan2017axiomatic}     & 0.1589 & 0.2560 & 0.0971 & 0.2404 & 0.3093 & 0.0689 & 0.1916 & 0.3073 & 0.1158 \\ 
\cellcolor{white} & \textbf{BIG} \cite{xu2020attribution}            & \underline{0.3122} & \underline{0.3504} & \underline{0.0382} & \underline{0.3258} & 0.3549 & \textbf{0.0291} & \underline{0.3473} & \underline{0.3956} & \underline{0.0482} \\ 
\rowcolor{palegray}
\cellcolor{white} & \textbf{EIG} \cite{jha2020enhanced}              & 0.1500 & 0.2533 & 0.1033 & 0.2327 & 0.3047 & 0.0720 & 0.1911 & 0.3091 & 0.1180 \\ 
\cellcolor{white} & \textbf{MIG} \cite{zaher2024manifold}            & 0.1596 & 0.2556 & 0.096 & 0.2349 & 0.3040 & 0.0691 & 0.1940 & 0.3104 & 0.1164 \\ 

\rowcolor{palegray}
\cellcolor{white} & \textbf{GIG} \cite{kapishnikov2021guided}        & 0.2507 & 0.3142 & 0.0636 & 0.2813 & 0.3302 & 0.0489 & 0.3027 & 0.3789 & 0.0762 \\ 
\cellcolor{white} & \textbf{AGI} \cite{pan2021explaining}            & 0.1640 & 0.2553 & 0.0913 & 0.2260 & 0.2922 & 0.0662 & 0.2171 & 0.3096 & 0.0924 \\ 
\rowcolor{palegray}
\cellcolor{white} & \textbf{IG$^2$} \cite{zhuo2024ig2}               & 0.1824 & 0.2676 & 0.0851 & 0.3047 & \underline{0.3584} & 0.0538 & 0.2569 & 0.3447 & 0.0878 \\ 
\cellcolor{white} & \textbf{SAMP} \cite{zhang2024path}               & 0.2880 & 0.3278 & 0.0398 & 0.2811 & 0.3118 & \underline{0.0307} & 0.3089 & 0.3751 & 0.0662 \\ 
\cmidrule(lr){2-11}
\rowcolor{lightgray}
    \multirow{-11}{*}{\rotatebox[origin=c]{90}{\textbf{ImageNet2012}}} 
\cellcolor{white} & \textbf{SIG (ours)}                            & \textbf{0.3422} & \textbf{0.3800} & \textbf{0.0378} & \textbf{0.3736} & \textbf{0.4107} & 0.0371 & \textbf{0.3784} & \textbf{0.4264} & \textbf{0.0480} \\ 
\bottomrule
\end{tabular}
\end{table*}

\Cref{fig:path_vgg} dissects the attribution process to reveal why SIG outperforms baseline methods in semantic localization. 
IG fails to suppress high-frequency noise; as shown in the middle block, noisy gradients are continuously accumulated throughout the entire integration path, resulting in a scattered attribution map. 
While BIG also progresses from blurred to sharp features, it tends to produce attributions dominated by boundary artifacts. The blurred path inherently emphasizes high-contrast edges, causing gradients to activate primarily along the object's silhouette. 
In contrast, SIG demonstrates a distinct coarse-to-fine aggregation mechanism. A key insight is that SIG implicitly filters out early-stage noise, as gradients remain suppressed until the coarse spectral components form a recognizable semantic structure. Consequently, attribution accumulates only after the classifier perceives the object's identity, ensuring the final attribution map captures the relevant features of the object while ignoring irrelevant background noise. 
Additional examples are provided in Appendix~\ref{app:path}.

\begin{figure}[ht]
    \centering
    \includegraphics[width=\linewidth]{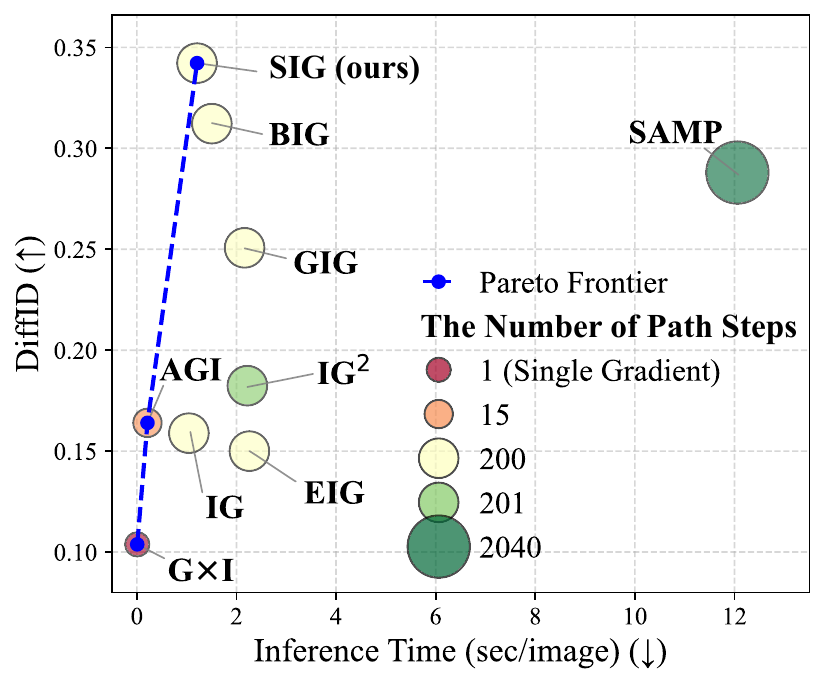}
    \caption{\textbf{Trade-off between explanation quality (DiffID) and inference time.} Results are evaluated on ImageNet2012 using ResNet18. SIG achieves the highest DiffID score while maintaining a Pareto-optimal position. 
    }
    \Description{
    A scatter plot comparing attribution methods based on explanation quality (DiffID, y-axis, higher is better) and efficiency (Inference Time, x-axis, lower is better).
    Circle sizes represent the number of path steps, ranging from 1 to 201.
    A blue dashed line indicates the Pareto Frontier, connecting GxI, AGI, and SIG. 
    SIG is positioned at the top-left, achieving the highest DiffID (0.34) with a low inference time (1.2 seconds), significantly outperforming more computationally expensive methods such as SAMP (12.0 seconds).}
    \label{fig:runtime}
\end{figure}

\subsection{Quantitative Evaluation}
\label{sec:quant}

\paragraph{\textbf{Faithfulness Metrics}}
\Cref{tab:quant} presents quantitative results across three datasets (Oxford-IIIT Pet, Oxford 102 Flower, and ImageNet2012) and three architectures (ResNet18, VGG16, and InceptionV1), evaluated on DiffID ($\uparrow$), Insertion ($\uparrow$), and Deletion ($\downarrow$). SIG consistently matches or outperforms existing methods across the majority of settings, particularly in DiffID and Insertion scores. On ImageNet, SIG achieves the best results in nearly all settings across three architectures and three metrics. Across the remaining datasets, SIG remains competitive with strong baselines such as BIG and SAMP, frequently ranking first or second. On Oxford 102 Flower, a fine-grained dataset where petal textures and small visual details are highly discriminative, blur-to-sharp paths such as BIG can be locally competitive on ResNet18 Deletion, yet SIG still attains the best DiffID on VGG16 and InceptionV1 and the best Deletion on InceptionV1, indicating that the coarse-to-fine path remains effective even when fine details dominate the decision. 
Overall, these results indicate that the coarse-to-fine integration path of SIG yields faithful attributions across diverse datasets and model architectures.

\paragraph{\textbf{Leakage-Aware Perturbation (ROAD)}}
Standard insertion / deletion protocols are known to be susceptible to information leakage through the shape of the perturbation mask~\citep{rong2022consistent}. To address this directly, we additionally evaluate ROAD~\citep{rong2022consistent} on ImageNet using a fixed imputer, which is explicitly designed to mitigate this leakage. SIG attains the best GAP score on all three backbones (ResNet18 0.2673, VGG16 0.3011, Inception 0.3104), outperforming the second-best method BIG on every classifier; the full table with MoRF/LeRF decompositions appears in Appendix~\ref{app:road}.

\paragraph{\textbf{Non-Perturbation Localization}}
As an orthogonal validation that does not rely on input perturbation, we evaluate spatial localization against ImageNet ground-truth bounding boxes and class-agnostic SAM~2 pseudo-masks~\citep{ravi2025sam}, and report Pointing Game~\citep{zhang2018top} accuracy. SIG attains the best mean BBox IoU on ResNet18 and InceptionV1 and the second-best on VGG16, the best Pointing Game accuracy on average (0.6694 vs.\ 0.6476 for BIG), and the best mean SAM~2 IoU (0.2315 vs.\ 0.2171 for BIG). Detailed numbers are reported in Appendix~\ref{app:localization}. Together with the perturbation-based and leakage-aware results, these localization measurements provide a triangulated empirical case for SIG's improved faithfulness.

\paragraph{\textbf{RemOve And Retrain (ROAR)}}
To further validate the robustness of our method, we conduct the ROAR~\citep{hooker2019benchmark} evaluation on the CIFAR-10 dataset. Unlike standard perturbation metrics, ROAR involves retraining the model on the modified dataset, thereby providing a more rigorous assessment of feature importance by mitigating the impact of distribution shifts. 
\Cref{fig:roar} illustrates the test accuracy as a function of the pixel removal ratio. We observe that SIG (cyan line) induces the sharpest degradation in model accuracy compared to all other baselines, including strong competitors like BIG and SAMP. 
For instance, at a 50\% removal ratio, the accuracy of the model trained on SIG-processed data drops to approximately 30\%, significantly lower than other methods. 
Consequently, SIG achieves the lowest area under the curve (AUC), demonstrating that it most accurately identifies the discriminative features essential for the model's decision-making process.

\subsection{Runtime and Efficiency Analysis}
\label{sec:runtime}

We evaluate the trade-off between explanation quality and computational cost to demonstrate the practical feasibility of our method. 
To ensure precise benchmarking, all inference times were measured on NVIDIA B200 GPUs using a batch size of 1. 
As shown in \Cref{fig:runtime}, SIG achieves the highest DiffID score, significantly outperforming baseline methods while maintaining a Pareto-optimal position. 
Notably, our method introduces only a marginal computational overhead compared to Standard IG. 
Specifically, the average inference time for SIG is $1.2085 \pm 0.0943$ seconds per image, which corresponds to an approximate $16\%$ increase over IG ($1.0450 \pm 0.0853$ seconds). 
This minimal increase is consistent with the theoretical complexity of 
$O(D^{1.5} + M \cdot C_f)$
, where the additional term $D^{1.5}$ from the spectral path reconstruction is negligible compared to the cost of model backpropagation ($C_f$), as detailed in Appendix~\ref{app:runtime}. 
These results confirm that SIG effectively incorporates spectral components to boost explanation quality without incurring prohibitive computational costs.

\section{Related Work}
\label{sec:related_work}

\subsection{Attribution Methods}

Feature attribution methods explain neural network predictions by assigning importance scores to input features. These methods fall into two broad categories based on their access requirements. Gradient-based methods leverage internal model information through backpropagation. The saliency method \citep{simonyan2014saliency} computes the gradient of the output with respect to input pixels, while Grad$\times$Input \citep{shrikumar2016not} multiplies gradients with input values for better alignment. Guided Backpropagation \citep{springenberg2014striving} filters negative gradients during backward passes, and Layerwise Relevance Propagation \citep{bach2015pixel} redistributes relevance scores through network layers using propagation rules. SmoothGrad \citep{smilkov2017smoothgrad} reduces noise by averaging gradients over multiple noise-perturbed inputs. Perturbation-based methods, in contrast, require only input-output access. LIME \citep{ribeiro2016why} fits local surrogate models on masked inputs, while RISE \citep{petsiuk2018rise} aggregates predictions from randomly masked inputs to estimate importance. Sobol attribution \citep{fel2021look} employs variance-based sensitivity analysis for feature importance estimation.

\subsection{Path-based Attribution Methods}

Integrated Gradients (IG) \citep{sundararajan2017axiomatic} computes attributions by accumulating gradients along a path connecting a baseline to the input, satisfying axiomatic properties including completeness and sensitivity. However, the choice of integration path significantly affects attribution quality, motivating substantial research into alternative path constructions. Guided Integrated Gradients (GIG) \citep{kapishnikov2021guided} greedily selects features with small gradient magnitudes at each step, avoiding high-curvature regions that contribute noisy gradients. Blur Integrated Gradients (BIG) \citep{xu2020attribution} defines paths through progressive Gaussian deblurring, adhering to scale-space axioms. Iterative Gradient path Integrated Gradient (IG$^2$) \citep{zhuo2024ig2} optimizes paths to follow steeper gradients, avoiding saturation regions. Adversarial Gradient Integration (AGI) \citep{pan2021explaining} constructs paths through adversarial examples using loss minimization. Salient Manipulation Path (SAMP) \citep{zhang2024path} formulates path selection as an optimization problem under a concentration principle that allocates attributions to the most salient features. Stick-breaking Path Integration (SPI) \citep{jeon2023beyond} samples stochastic paths from a stick-breaking process and aggregates attributions across multiple paths to reduce variance. Denoising Diffusion Path (DDPath) \citep{lei2024denoising} leverages diffusion models to construct paths aligned with the data distribution, ensuring intermediate samples remain realistic. Manifold-based approaches including Enhanced Integrated Gradients (EIG) \citep{jha2020enhanced} and Manifold Integrated Gradients (MIG) \citep{zaher2024manifold} constrain paths to the latent space of generative models, ensuring intermediate points remain on the data distribution. Manifold-Aligned Guided Integrated Gradients (MA-GIG) \citep{kim2026manifold} further combines GIG's adaptive feature selection with manifold-aware path construction to jointly avoid high-curvature regions and stay close to the data manifold. In contrast, our method differs from these approaches by leveraging singular value decomposition to construct paths that prioritize structurally important components, requiring neither generative models nor iterative optimization.

\section{Conclusion}

We proposed Spectral Integrated Gradients (SIG), a path-based attribution method that leverages singular value decomposition to construct integration paths respecting the spectral hierarchy of visual information. By activating dominant structural components before fine-grained details, SIG follows a principled coarse-to-fine progression that suppresses the accumulation of erratic gradients during early integration steps. A continuous relaxation with an overlap parameter enables smooth transitions between spectral components and recovers standard IG as a special case. Extensive experiments across multiple image classification benchmarks and architectures demonstrated that SIG consistently produces cleaner attribution maps and achieves improved quantitative performance on perturbation-based (DiffID, ROAR), leakage-aware (ROAD), and non-perturbation localization (BBox IoU, Pointing Game, SAM~2 IoU) evaluations compared to existing path-based methods, while introducing only marginal computational overhead beyond standard Integrated Gradients.

\paragraph{\textbf{Limitations.}}
SIG inherits the standard IG assumption of a fixed baseline (typically a zero image), which may limit attribution quality for inputs where the baseline choice is ambiguous. The rank-constrained optimization criterion in Eq.~\eqref{eq:optimization} justifies the path construction objective but does not by itself imply causal or semantic faithfulness of the resulting attributions; faithfulness is supported empirically through the triangulated evaluations above, not by a formal guarantee. The correspondence between singular value magnitude and semantic importance is also empirical rather than universal: it holds well for natural images, where SVD's variance ordering aligns with spatial structure, but a pilot study on text classification (Appendix~\ref{app:text_domain}) shows that SIG approaches IG only when the overlap parameter is large, indicating that the coarse-to-fine inductive bias is most beneficial for image inputs. Extending SIG to adaptive baselines, joint multi-channel decompositions, or alternative spectral families that better capture domain-specific structure represents a promising direction for future work. Because SIG modifies only the integration path, the same construction can in principle be combined with other Aumann--Shapley-style attribution rules beyond IG.

\begin{acks}
This work was supported in part by the Institute for Information \& Communications Technology Planning \& Evaluation (IITP) grants funded by the Korean government (MSIT): RS-2019-II190075 (Artificial Intelligence Graduate School Program, KAIST), RS-2022-II220984 (Development of Artificial Intelligence Technology for Personalized Plug-and-Play Explanation and Verification of Explanation), and RS-2024-00457882 (AI Research Hub Project). This work was also supported by the InnoCORE program of MSIT (N10250156), the AI Computing Infrastructure Enhancement (GPU Rental Support) User Support Program funded by MSIT (RQT-25-120227).
\end{acks}



\bibliographystyle{ACM-Reference-Format}
\bibliography{main}

\appendix
\crefalias{section}{appendix}

\section*{Appendix Contents}
\begingroup
\setlength{\parindent}{0pt}
\newcommand{\tocsec}[3]{\noindent\hyperref[#3]{\makebox[1.5em][l]{\textbf{#1}}\textbf{#2}}\dotfill\hyperref[#3]{\pageref{#3}}\par}
\newcommand{\tocsub}[3]{\noindent\hspace{1.5em}\hyperref[#3]{\makebox[2.0em][l]{#1}#2}\dotfill\hyperref[#3]{\pageref{#3}}\par}
\tocsec{A}{Proofs}{appendix:proofs}
\tocsec{B}{Experimental Setup}{app:setup}
\tocsec{C}{Computational Cost \& Runtime Analysis}{app:runtime}
\tocsub{C.1}{Computational Cost \& Resources}{app:runtime}
\tocsub{C.2}{Efficient Runtime in Faithfulness Metrics}{app:runtime_eff}
\tocsec{D}{Additional Experimental Results}{app:add}
\tocsub{D.1}{Extended Frequency-Domain Analysis}{app:freq}
\tocsub{D.2}{Extended Path Analysis}{app:path}
\tocsub{D.3}{Additional Qualitative Results}{app:qual}
\tocsec{E}{Hyperparameter Sensitivity Analysis}{app:quant}
\tocsec{F}{Baseline Sensitivity}{app:baseline_sensitivity}
\tocsec{G}{Joint vs.\ Per-Channel SVD}{app:joint_svd}
\tocsec{H}{Logits vs.\ Softmax Probabilities}{app:logits}
\tocsec{I}{Alternative Spectral Decompositions}{app:decomposition}
\tocsec{J}{Gating Schedule Ablation}{app:gating}
\tocsec{K}{High-Resolution Scalability}{app:scalability}
\tocsec{L}{Leakage-Aware Perturbation: ROAD}{app:road}
\tocsec{M}{Non-Perturbation Localization}{app:localization}
\tocsec{N}{Generalization Beyond Images}{app:text_domain}
\tocsec{O}{Implementation Code}{app:code}
\endgroup
\vspace{0.1cm}

\section{Proofs}\label{appendix:proofs}

\begin{proposition}[Geometric Symmetry Preservation]
Let $T$ be a permutation matrix representing a discrete spatial operation (e.g., $90^\circ$-multiple rotation or reflection).
Let $f:\mathbb{R}^n\to\mathbb{R}$ satisfy $f(Tx)=f(x)$ for all $x$.
Then SIG is equivariant under $T$:
\[
\mathcal{A}_{SIG}(Tx',Tx)=T\,\mathcal{A}_{SIG}(x',x).
\]
\end{proposition}

\begin{proof}
Let $\Delta = x-x'$ and write its SVD as $\Delta = U\Sigma V^\top$.
For transformed inputs, $\Delta_T = Tx-Tx' = T\Delta$.
Since $T$ is orthogonal, $TU$ is also orthogonal, so $(TU)\Sigma V^\top$ is a valid SVD of $\Delta_T$ with the same singular values $\sigma_i$.

Define the SIG path
\[
\gamma(\alpha)=x' + \sum_i \phi_i(\alpha; \omega)\sigma_i u_i v_i^\top,
\quad
\gamma_T(\alpha)=Tx' + \sum_i \phi_i(\alpha; \omega)\sigma_i (Tu_i) v_i^\top.
\]
Then $\gamma_T(\alpha)=T\gamma(\alpha)$ for all $\alpha$, and differentiating gives $\gamma_T'(\alpha)=T\gamma'(\alpha)$.

Next, from $f(Tx)=f(x)$ and the chain rule,
\[
\nabla f(Tx) = T\nabla f(x) \qquad (\text{since } T^{-T}=T \text{ for orthogonal }T).
\]
Thus $\nabla f(\gamma_T(\alpha)) = \nabla f(T\gamma(\alpha)) = T\nabla f(\gamma(\alpha))$.

Finally, SIG attribution is the path integral
\[
\mathcal{A}_{SIG}(x',x)=\int_0^1 \nabla f(\gamma(\alpha))\odot \gamma'(\alpha)\,d\alpha.
\]
Therefore,
\begin{align*}
\mathcal{A}_{SIG}(Tx',Tx)
&=\int_0^1 \nabla f(\gamma_T(\alpha))\odot \gamma_T'(\alpha)\,d\alpha\\
&=\int_0^1 \bigl(T\nabla f(\gamma(\alpha))\bigr)\odot \bigl(T\gamma'(\alpha)\bigr)\,d\alpha.
\end{align*}
Since $T$ is a permutation matrix, it commutes with element-wise products:
$(Ta)\odot(Tb)=T(a\odot b)$.
Hence
\[
\mathcal{A}_{SIG}(Tx',Tx)
= T\int_0^1 \nabla f(\gamma(\alpha))\odot \gamma'(\alpha)\,d\alpha
= T\,\mathcal{A}_{SIG}(x',x).
\]
\vspace{-0.2cm}
\end{proof}

\section{Experimental Setup}
\label{app:setup}


We describe the implementations and hyperparameters used for reproducibility.  Hyperparameters follow official implementations or recommendations from the respective papers and are applied consistently across all datasets and classifiers.

\begin{itemize}[leftmargin=*]
    \item \textbf{G$\times$I (Gradient $\times$ Input)}~\citep{shrikumar2016not}: Computes attribution by element-wise multiplication of the input gradient with the input values.
    
    \item \textbf{Integrated Gradients (IG)}~\citep{sundararajan2017axiomatic}: We use 200 integration steps along a straight-line path with a zero baseline.
    
    \item \textbf{Guided IG (GIG)}~\citep{kapishnikov2021guided}: We use 200 integration steps with a feature selection fraction of 0.1 (10\% of features updated per step) and a zero baseline.
    
    \item \textbf{IG$^2$}~\citep{zhuo2024ig2}: We use 201 steps with a step size of 0.02 and random reference mode.
    
    \item \textbf{Adversarial Gradient Integration (AGI)}~\citep{pan2021explaining}: We use 15 maximum iterations, 10 negative classes, and a step size of 0.05.
    
    \item \textbf{Blur IG (BIG)}~\citep{xu2020attribution}: Integrates gradients along a Gaussian blur path in scale-space. We use 200 integration steps with maximum blur standard deviation $\sigma_{\max} = 35$.
    
    \item \textbf{SAMP}~\citep{zhang2024path}: Searches for paths that satisfy the concentration principle. We use step size 5, fragment count $n_{\text{frag}} = 5$, blur kernel length $k_{\text{len}} = 11$, kernel sigma $k_{\sigma} = 5$, bidirectional path traversal, and blur as the insertion baseline.
    
    \item \textbf{Enhanced IG (EIG)}~\citep{jha2020enhanced}: Performs linear interpolation in the VAE latent space with 200 integration steps and a zero baseline.
    
    \item \textbf{Manifold IG (MIG)}~\citep{zaher2024manifold}: Computes geodesic paths in the VAE latent space with 200 integration steps and a zero baseline. We use a learning rate of $5 \times 10^{-5}$ for geodesic optimization with 2 iterations.
    
    \item \textbf{Spectral IG (Ours)}: We use 200 integration steps with an overlap parameter $\omega = 0.4$ and a zero baseline.
\end{itemize}

\section{Computational Cost \& Runtime Analysis}
\label{app:runtime}

\subsection{Computational Cost \& Resources}
All experiments were conducted on a system equipped with four NVIDIA B200 GPUs. The computational complexity of our proposed module is $O(D^{1.5} + M \cdot C_f)$, where $M$ denotes the number of integration steps, $D$ represents the input dimension (e.g., total number of pixels), and $C_{f}$ accounts for the computational cost of a single backward pass of a classifier. To ensure reproducibility, we measured the inference time, which consumed approximately 0.17 GPU hours for computing attributions of 500 samples. We utilized a batch size of 1 to compute gradients without precision loss.

\subsection{Efficient Runtime in Faithfulness Metrics} 
\label{app:runtime_eff}

\begin{figure}[ht]
    \centering
    \begin{subfigure}{\linewidth}
        \centering
        \includegraphics[width=0.97\linewidth]{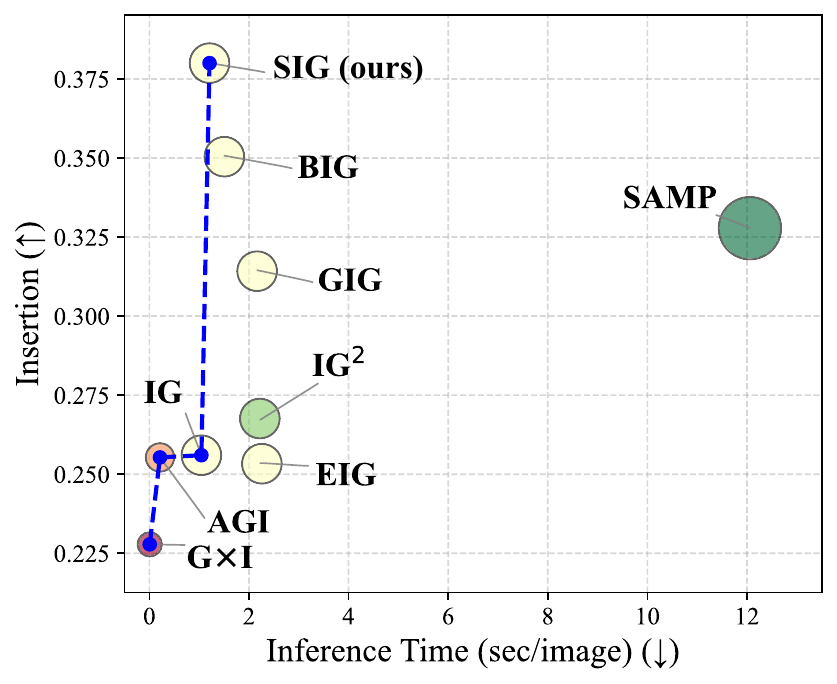}
        \caption{Performance on Insertion Game}
    \end{subfigure}
    \vspace{0.2cm} 
    \begin{subfigure}{\linewidth}
        \centering
        \includegraphics[width=0.97\linewidth]{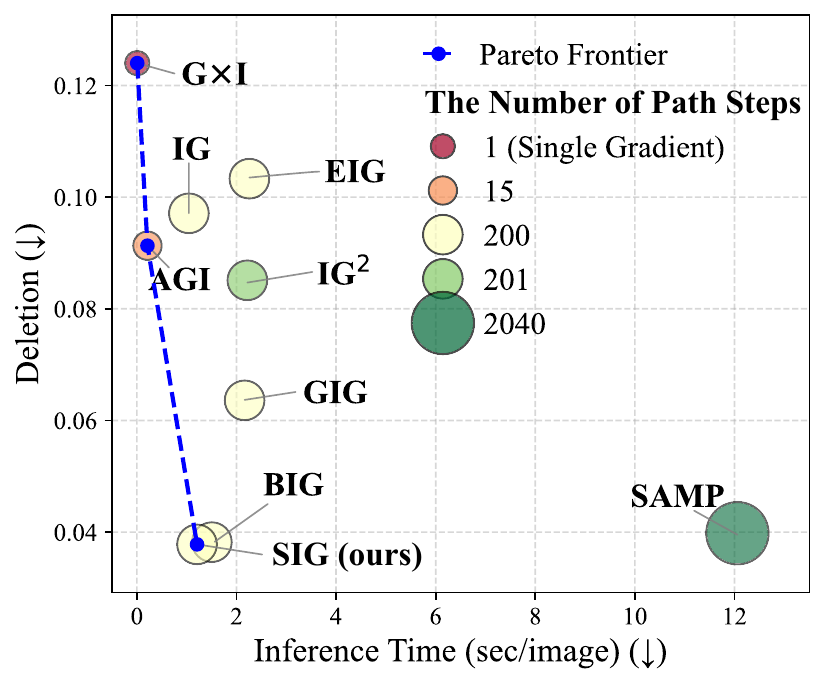}
        \caption{Performance on Deletion Game}
    \end{subfigure}
    \caption{\textbf{Extended runtime analysis on faithfulness metrics.} We report the trade-off curves for (a) Insertion and (b) Deletion games. SIG consistently occupies the optimal region (top-left for Insertion, bottom-left for Deletion) across both benchmarks, demonstrating robust fidelity with minimal latency increase.}
    \label{fig:runtime_full} 
    \vspace{-0.2cm}
    \Description{
    The figure consists of two scatter plot charts comparing XAI methods across two benchmarks: (a) Insertion Game (higher is better) and (b) Deletion Game (lower is better), both plotted against Inference Time (x-axis, lower is better). 
    1. Axes and Legend: The x-axis represents Inference Time in seconds per image. The y-axis represents the game score. Circle sizes indicate the 'Number of Path Steps', with SIG using 200 steps. A blue dashed line marks the Pareto Frontier.
    2. Performance in Insertion Game (a): SIG reaches the highest Insertion score (0.38) at 1.2 seconds, outperforming BIG (0.35, 1.51 seconds), GIG (0.31, 2.16 seconds), and IG (0.26, 1.05 seconds) while remaining significantly faster than SAMP (0.33, 12 seconds).
    3. Performance in Deletion Game (b): SIG achieves the best (lowest) Deletion score (0.038) at the same inference time, forming the optimal point on the Pareto Frontier along with GxI (0.12, 0.008 seconds) and AGI (0.09, 0.21 seconds).
    4. Summary: Across both games, SIG consistently occupies the top-left (for Insertion) or bottom-left (for Deletion) optimal regions, demonstrating superior explanation quality with minimal computational overhead compared to the standard IG.
    }
\end{figure}

We present an extended runtime analysis of baseline methods in conjunction with faithfulness metrics. Inference times were averaged across ten samples. A key distinction in our experimental setup concerns the SAMP algorithm, where the number of steps is determined adaptively by concurrently traversing two paths: one for removing influential features and another for inserting them. Consequently, for this analysis, we defined the total path length for SAMP as the sum of the steps taken in both the deletion and insertion paths.

\Cref{fig:runtime_full} illustrates the trade-off between computational efficiency (inference time per image) and faithfulness performance (Insertion and Deletion scores), with the number of path steps provided as a reference (represented by circle size). A critical observation is the substantial computational overhead associated with SAMP. As shown in the plots, SAMP requires approximately 12 seconds per image-nearly six times slower than other path-based methods such as GIG or SIG. This high latency is directly attributed to its adaptive mechanism, where the total number of path steps accumulates to 2,040 on average, significantly exceeding the fixed budget (e.g., 200 steps) utilized by baseline methods.

In contrast, SIG demonstrates a superior Pareto-optimal position between fidelity and runtime. Despite requiring significantly lower inference time (approximately 1.2 seconds per image), SIG outperforms the computationally expensive SAMP in both benchmarks. Specifically, SIG consistently occupies the optimal regions: the top-left in the Insertion game (\Cref{fig:runtime_full}a) and the bottom-left in the Deletion game (\Cref{fig:runtime_full}b). This indicates that the massive step count and resulting latency of SAMP do not translate into performance gains over SIG. These results validate that SIG provides robust fidelity with minimal latency, making it a far more practical choice for large-scale attribution tasks.

\section{Additional Experimental Results}
\label{app:add}

In this section, we provide supplementary experimental results that further validate the effectiveness and robustness of SIG.

\subsection{Extended Frequency-Domain Analysis}
\label{app:freq}

We provide extended frequency-domain visualizations in \Cref{fig:frequency_appx} to further substantiate the spectral properties of the integration path. Consistent with the main text analysis, we compare the 2D Discrete Fourier Transform (DFT) magnitude spectrum of the input difference $\Delta(\alpha)$ across integration steps $\alpha \in \{0, 0.25, 0.5, 0.75, 1\}$.


\begin{figure*}[t]
    \centering
    \begin{subfigure}[b]{0.443\linewidth}
        \centering
        \includegraphics[width=\linewidth]{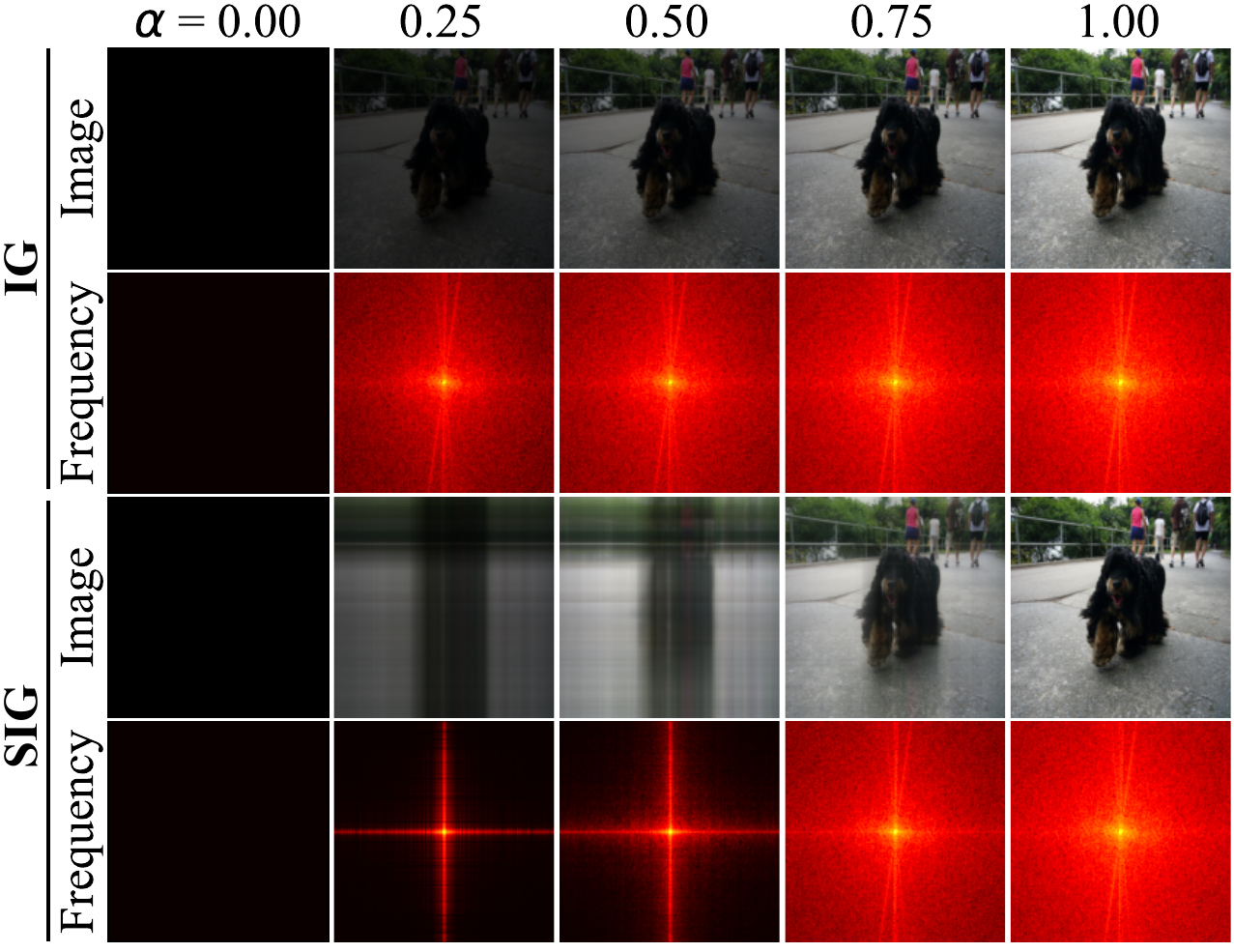}
    \end{subfigure}
    \hspace{1em}
    \begin{subfigure}[b]{0.443\linewidth}
        \centering
        \includegraphics[width=\linewidth]{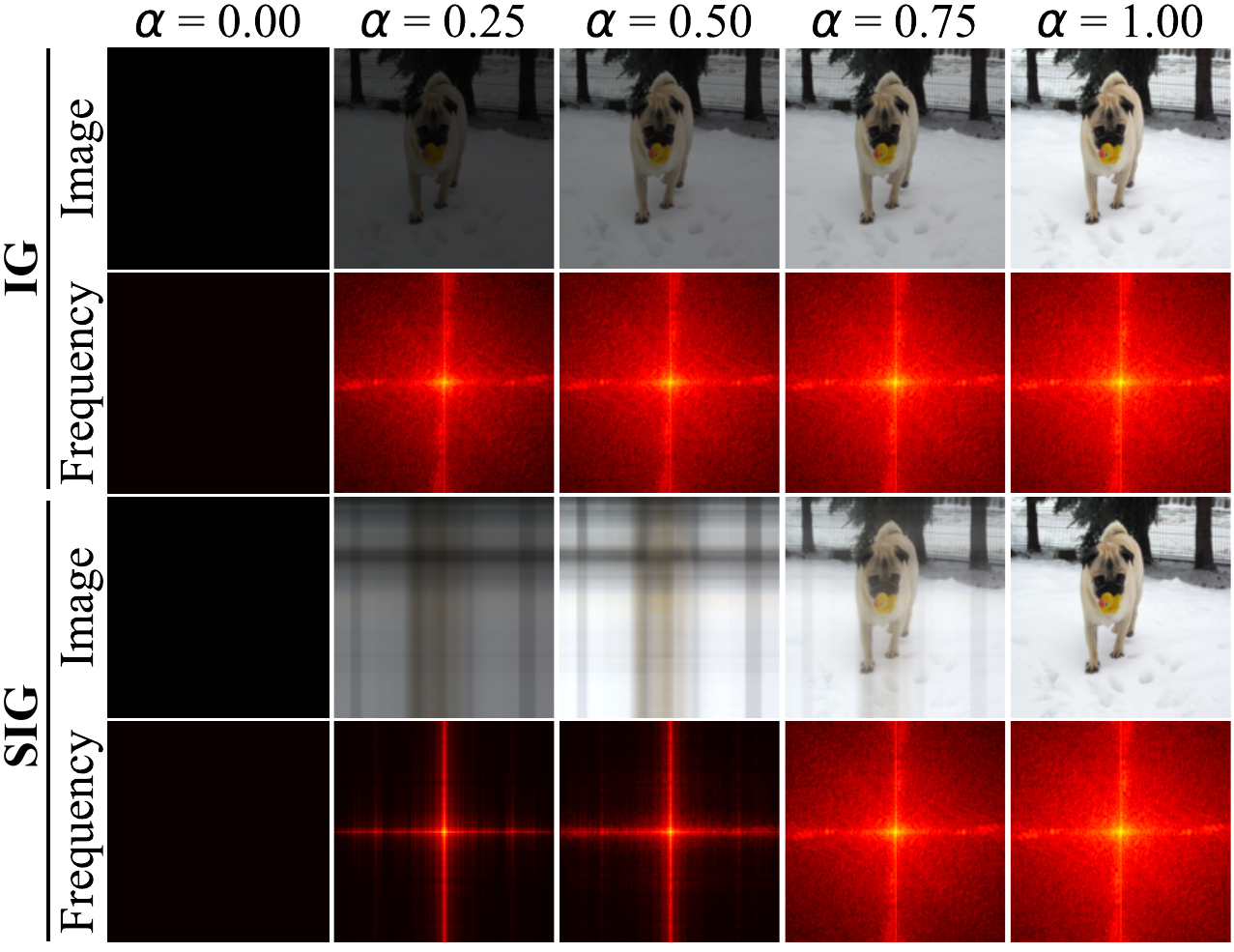}
    \end{subfigure}
    \vspace{-0.1cm}
    \caption{\textbf{Additional examples of frequency-domain analysis of integration paths.} Supplementary to \Cref{fig:frequency}, these examples further validate the distinct behavior of the integration paths. While IG (top) uniformly introduces all frequencies, SIG (bottom) consistently reconstructs the image from low-frequency structures to high-frequency details.}
    \label{fig:frequency_appx}
    \vspace{-0.2cm}
\end{figure*}

\begin{figure*}[ht]
    \centering
    \begin{subfigure}[b]{0.443\linewidth}
        \centering
        \includegraphics[width=\linewidth]{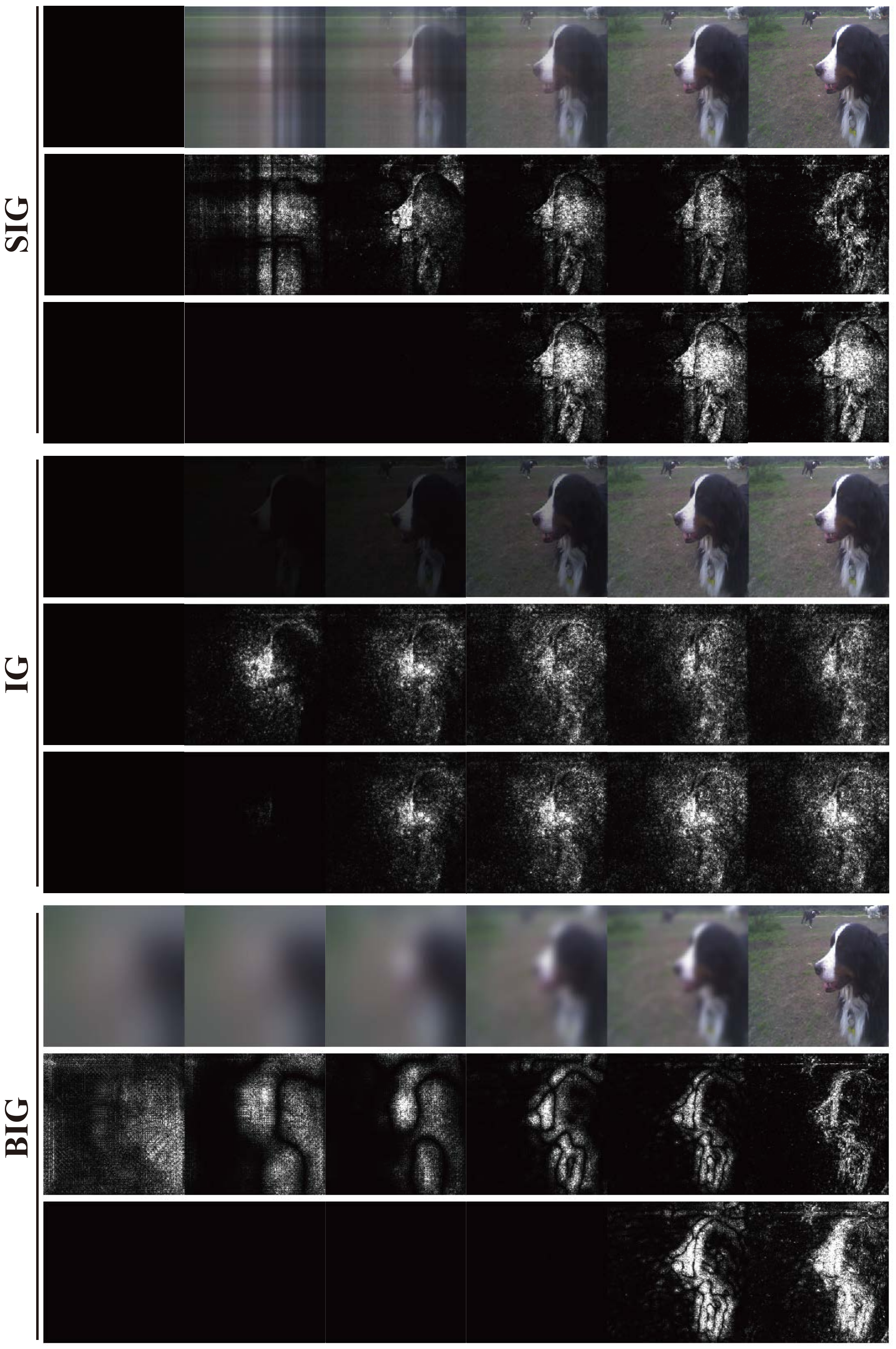}
        \caption{Predicted: `Bernese Mountain Dog' (0.907 confidence)}
        \label{fig:path_resnet_dog}
    \end{subfigure}
    \hspace{2em}
    \begin{subfigure}[b]{0.443\linewidth}
        \centering
        \includegraphics[width=\linewidth]{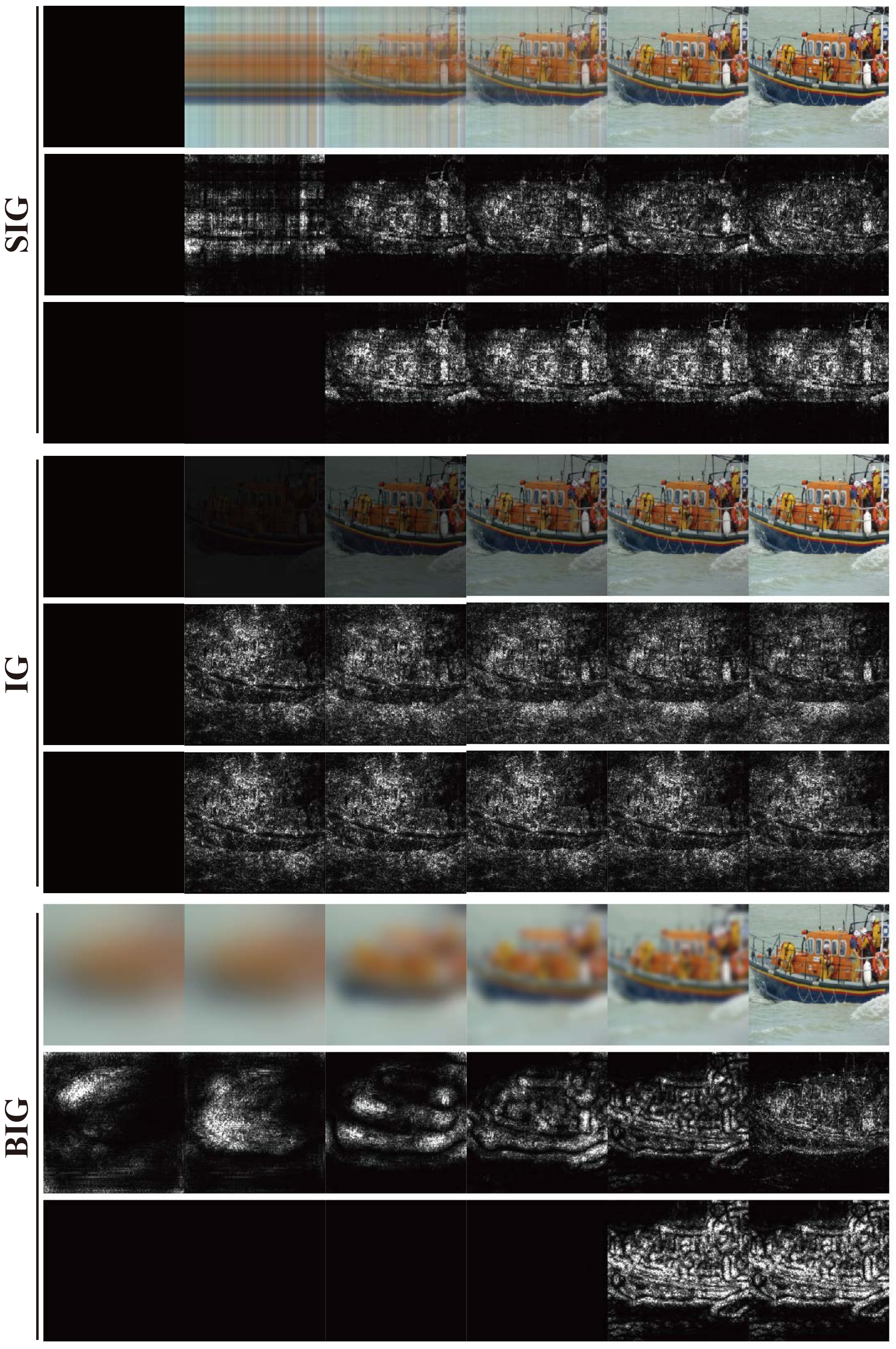}
        \caption{Predicted: `Lifeboat' (1.000 confidence)}
        \label{fig:path_resnet_scorpion}
    \end{subfigure}
    \caption{\textbf{Additional path analysis examples.} Supplementary to \Cref{fig:path_vgg}, these results on ResNet18 (left) and VGG16 (right) further validate our observations. The visualization setup follows the same convention as the main text. Consistent with previous results, SIG effectively suppresses early-stage noise and boundary artifacts compared to IG and BIG.}
    \label{fig:path_resnet_combined}
    \vspace{-0.2cm}
    \Description{Two side by side path analysis panels on ImageNet. The left panel uses ResNet18 on an image classified as Bernese Mountain Dog, and the right panel uses VGG16 on an image classified as Lifeboat. Each panel compares three methods (IG, BIG, and SIG) across three rows: interpolated path images, per step gradient contributions, and cumulative attribution maps at 10 evenly spaced integration steps. In both examples, IG produces noisy gradients throughout all steps, BIG concentrates gradients along edges, and SIG suppresses gradients during early steps and accumulates attribution only after coarse structure is established, yielding cleaner and more semantically focused attribution maps.}
\end{figure*}


\subsection{Extended Path Analysis}
\label{app:path}

In~\Cref{fig:path_resnet_combined}, we present additional path analysis results. These figures provide further comparisons of our SIG against the baseline path-based methods of IG and BIG on test samples from the ImageNet dataset using ResNet18 classifier.

\subsection{Additional Qualitative Results}
\label{app:qual}

We provide extended qualitative comparisons across all three datasets: ImageNet 2012 (\Cref{fig:qual_imagenet}), Oxford 102 Flower (\Cref{fig:qual_oxfordflower}), and Oxford-IIIT Pet (\Cref{fig:qual_oxfordpet}). Each figure presents attribution maps from ten methods evaluated on ResNet18, VGG16, and InceptionV1 classifiers. Across all settings, SIG consistently produces attribution maps that are more focused on semantically relevant regions while suppressing background clutter.

\section{Hyperparameter Sensitivity Analysis}
\label{app:quant}

\Cref{tab:quant_full} provides an extended quantitative evaluation of SIG across varying overlap parameters $\omega \in \{0.1, 0.2, \dots, 0.6\}$. While the main text reports results for $\omega=0.4$, this analysis demonstrates the method's sensitivity to the overlap parameter and confirms the robustness of the chosen configuration. 

The results indicate that performance is generally stable for $\omega \in [0.3, 0.5]$, with $\omega=0.4$ often yielding optimal or near-optimal scores across datasets and architectures. 

\paragraph{\textbf{Sensitivity Analysis.}} 
Lower overlap values $\omega=0.1, 0.2$ tend to produce slightly weaker results in DiffID and Insertion metrics compared to $\omega=0.4$, suggesting that a moderate degree of overlap is beneficial for smoothing the transition between spectral components. Conversely, increasing the overlap to $\omega=0.6$ leads to a noticeable performance drop (e.g., DiffID decreases from 0.5210 to 0.3631 on Oxford-IIIT Pet with ResNet18), likely because the path becomes too similar to the linear interpolation of standard IG, reintroducing noisy gradients.

\paragraph{\textbf{Robustness.}} 
Despite these variations, SIG with $\omega=0.4$ consistently outperforms or remains highly competitive against baseline methods like BIG and SAMP, validating it as a robust default setting for diverse image classification tasks. 
This ablation study confirms that the proposed spectral path construction is effective and that $\omega=0.4$ strikes a good balance between sequentially prioritizing structural information and maintaining path continuity.

\section{Baseline Sensitivity}
\label{app:baseline_sensitivity}

We evaluate the sensitivity of SIG and representative competitors to the choice of baseline image on ImageNet (ResNet18, 500 samples). We consider three commonly used baselines: a zero image (default), a per-channel mean image, and a Gaussian-blurred copy of the input. Table~\ref{tab:baseline_sensitivity} reports DiffID for each baseline. SIG attains the highest mean DiffID across the three baselines while exhibiting the smallest range, indicating that the spectral path is robust to the baseline choice. IG, in particular, is highly baseline-sensitive (range $0.1471$), which is consistent with its straight-line path placing all of the burden on the baseline-to-input difference.

\begin{table}[ht]
\centering
\caption{\textbf{Baseline sensitivity on ImageNet (ResNet18, DiffID $\uparrow$).} For each method, ``mean$_3$'' is the average DiffID across the three baselines (zero, mean, blur), and ``range'' is the max$-$min spread of DiffID over the same three baselines (smaller is more robust).}
\vspace{-0.1cm}
\label{tab:baseline_sensitivity}
\setlength{\tabcolsep}{6pt}
\begin{tabular}{lcccccc}
\toprule
\textbf{Method} & \textbf{zero} & \textbf{mean} & \textbf{blur} & \textbf{mean$_3$} & \textbf{range} \\
\midrule
IG  & 0.1564 & \underline{0.2720} & \underline{0.3036} & 0.2440 & 0.1471 \\
GIG & \underline{0.2540} & \textbf{0.3320} & 0.2884 & \underline{0.2915} & \underline{0.0780} \\
\rowcolor{lightgray}
SIG & \textbf{0.3453} & 0.3247 & \textbf{0.3138} & \textbf{0.3279} & \textbf{0.0316} \\
\bottomrule
\end{tabular}
\vspace{-0.2cm}
\end{table}

\section{Joint vs.\ Per-Channel SVD}
\label{app:joint_svd}

By default, SIG computes the SVD of the per-channel difference matrix $\Delta \in \mathbb{R}^{H\times W}$. An alternative is to reshape the multi-channel difference as a matrix $\Delta_{\text{joint}} \in \mathbb{R}^{C\cdot H \times W}$ and perform a single SVD over channels jointly, which can in principle exploit cross-channel correlations. Across the three datasets and three classifiers (9 settings), the joint variant yields a small mean DiffID improvement of $+0.0046$ on average, exceeding per-channel in $6/9$ settings. The improvement is marginal, and per-channel SVD is both simpler to implement and preserves the standard SIG formulation; we therefore retain per-channel SVD as the default.

\section{Logits vs.\ Softmax Probabilities}
\label{app:logits}

Following the original IG formulation~\citep{sundararajan2017axiomatic}, we compute attributions with respect to the post-softmax class probability $\sigma(f(x))_y$. As a sanity check, we re-run SIG using the raw class logit $f(x)_y$. The logit-based variant degrades DiffID in all $9$ settings (mean $\Delta\text{DiffID}=-0.1287$), without altering the rank ordering of methods. This is consistent with prior reports that probability-based scores yield more stable path integrals because saturated logits produce vanishingly small or unstable gradients.

\section{Alternative Spectral Decompositions}
\label{app:decomposition}

To isolate the contribution of the data-adaptive singular-component ordering, we replace the SVD path with three classical multi-scale alternatives that share a similar coarse-to-fine intuition: a discrete cosine transform (DCT) path, a wavelet-pyramid path, and a Laplacian-pyramid path. All variants use the same continuous overlap gating and identical hyperparameters. Table~\ref{tab:alt_paths} reports the average DiffID across the three datasets and three classifiers (9 settings each). SVD outperforms all three fixed-basis alternatives, suggesting that the data-adaptive nature of singular components, rather than the coarse-to-fine progression alone, is what drives the gain.

\begin{table}[ht]
\centering
\vspace{-0.2cm}
\caption{\textbf{Alternative spectral paths.} Mean DiffID is averaged over three datasets and three classifiers.}
\label{tab:alt_paths}
\vspace{-0.1cm}
\setlength{\tabcolsep}{8pt}
\begin{tabular}{lc}
\toprule
\textbf{Decomposition} & \textbf{Mean DiffID} ($\uparrow$) \\
\midrule
DCT & 0.2726 \\
Wavelet & 0.3346 \\
Laplacian & \underline{0.3507} \\
\rowcolor{lightgray}
SVD (SIG) & \textbf{0.4046} \\
\bottomrule
\end{tabular}
\vspace{-0.2cm}
\end{table}

\section{Gating Schedule Ablation}
\label{app:gating}

The continuous relaxation in Eq.~\eqref{eq:scaling_fn} uses a piecewise-linear ramp to gate each singular component. We compare against cosine, sigmoid, and step (hard-threshold) gating schedules with the same activation windows. As shown in Table~\ref{tab:gating}, linear and cosine perform nearly identically and sigmoid is marginally lower, while the hard step schedule incurs a noticeable drop, confirming that smooth gating is important for stable numerical integration.

\begin{table}[ht]
\centering
\vspace{-0.2cm}
\caption{\textbf{Gating schedule ablation.} Mean DiffID is averaged over three datasets and three classifiers.}
\label{tab:gating}
\vspace{-0.1cm}
\setlength{\tabcolsep}{8pt}
\begin{tabular}{lc}
\toprule
\textbf{Schedule} & \textbf{Mean DiffID} ($\uparrow$) \\
\midrule
\rowcolor{lightgray}
Linear (default) & \underline{0.4046} \\
Cosine & \textbf{0.4060} \\
Sigmoid & 0.3960 \\
Step & 0.3116 \\
\bottomrule
\end{tabular}
\vspace{-0.2cm}
\end{table}

\section{High-Resolution Scalability}
\label{app:scalability}

We measure the wall-clock cost of the SVD-based path construction at $256\times256$, $512\times512$, and $1024\times1024$ resolutions on a single NVIDIA B200 GPU (per-channel SVD, 200 integration steps, 5 samples each). Table~\ref{tab:scalability} reports the mean time for the SVD decomposition and for the full path reconstruction. Even at $1024\times1024$, the SVD itself takes less than one second per image, and the cost is dominated by classifier backward passes rather than by the spectral decomposition.

\begin{table}[ht]
\centering
\vspace{-0.2cm}
\caption{\textbf{SVD scalability (seconds, mean over 5 samples).}}
\label{tab:scalability}
\vspace{-0.1cm}
\setlength{\tabcolsep}{8pt}
\begin{tabular}{lcc}
\toprule
\textbf{Resolution} & \textbf{SVD decomp.} & \textbf{Path recon.} \\
\midrule
\rowcolor{lightgray}
$256\times256$ & \textbf{0.1192} & \textbf{0.4114} \\
$512\times512$ & \underline{0.2072} & \underline{0.4154} \\
$1024\times1024$ & 0.5485 & 0.4218 \\
\bottomrule
\end{tabular}
\vspace{-0.2cm}
\end{table}

\section{Leakage-Aware Perturbation: ROAD}
\label{app:road}
Standard insertion/deletion protocols can leak information through the shape and pattern of the removed pixels~\citep{rong2022consistent}. To directly mitigate this leakage, we evaluate ROAD~\citep{rong2022consistent} with a fixed imputer on ImageNet for all three classifiers. We report the GAP score together with its MoRF and LeRF components in Table~\ref{tab:road}. SIG attains the best GAP across all three classifiers and the best LeRF across all three. On MoRF, SIG is best on InceptionV1 and within $\sim0.001$--$0.003$ of BIG on the other two, which is expected since MoRF emphasises early removal of the most influential pixels where BIG's edge-heavy attributions are competitive. Each entry in Table~\ref{tab:road} is the mean over three independent random seeds controlling the ROAD imputer's noise and pixel-removal order; the corresponding standard deviations are uniformly $\le 0.003$ and hence omitted from the table.

\begin{table}[ht]
\centering
\vspace{-0.2cm}
\caption{\textbf{ROAD evaluation on ImageNet (fixed imputer).} GAP and LeRF higher is better; MoRF lower is better. Each entry is the mean over three random seeds for the ROAD imputer (standard deviations $\le 0.003$ across all entries and omitted for readability).}
\label{tab:road}
\vspace{-0.1cm}
\small
\setlength{\tabcolsep}{0.5pt}
\begin{tabular}{l ccc ccc ccc}
\toprule
& \multicolumn{3}{c}{\textbf{GAP} ($\uparrow$)} & \multicolumn{3}{c}{\textbf{MoRF} ($\downarrow$)} & \multicolumn{3}{c}{\textbf{LeRF} ($\uparrow$)} \\
\cmidrule(lr){2-4} \cmidrule(lr){5-7} \cmidrule(lr){8-10}
\textbf{Method} & R18 & VGG & Inc. & R18 & VGG & Inc. & R18 & VGG & Inc. \\
\midrule
G$\times$I & 0.1084 & 0.1591 & 0.1527 & 0.0875 & 0.0648 & 0.1033 & 0.1960 & 0.2239 & 0.2560 \\
IG & 0.1511 & 0.2068 & 0.2132 & 0.0685 & 0.0494 & 0.0837 & 0.2197 & 0.2561 & 0.2969 \\
BIG & \underline{0.2523} & 0.2771 & \underline{0.3027} & \textbf{0.0265} & \textbf{0.0209} & \underline{0.0410} & \underline{0.2788} & \underline{0.2980} & \underline{0.3436} \\
EIG & 0.1436 & 0.1883 & 0.1984 & 0.0744 & 0.0540 & 0.0912 & 0.2181 & 0.2423 & 0.2896 \\
MIG & 0.1501 & 0.1944 & 0.2121 & 0.0716 & 0.0499 & 0.0831 & 0.2217 & 0.2443 & 0.2952 \\
GIG & 0.1757 & 0.1959 & 0.2557 & 0.0364 & 0.0245 & 0.0466 & 0.2121 & 0.2204 & 0.3023 \\
AGI & 0.1126 & 0.1710 & 0.1894 & 0.0632 & 0.0504 & 0.0633 & 0.1758 & 0.2214 & 0.2527 \\
IG$^2$ & 0.1420 & \underline{0.2443} & 0.2181 & 0.0592 & 0.0335 & 0.0621 & 0.2012 & 0.2778 & 0.2802 \\
SAMP & 0.1199 & 0.1052 & 0.1873 & 0.0500 & 0.0368 & 0.0664 & 0.1699 & 0.1420 & 0.2537 \\
\rowcolor{lightgray}
SIG & \textbf{0.2673} & \textbf{0.3011} & \textbf{0.3104} & \underline{0.0273} & \underline{0.0234} & \textbf{0.0409} & \textbf{0.2946} & \textbf{0.3245} & \textbf{0.3513} \\
\bottomrule
\end{tabular}
\vspace{-0.2cm}
\end{table}

\section{Non-Perturbation Localization}
\label{app:localization}

Beyond perturbation-based metrics, we measure how well the attribution maps spatially localize the predicted object. We report three complementary measures on ImageNet: (i) mean intersection-over-union (IoU) of the top $15\%$ attribution mass against SAM~2~\citep{ravi2025sam} class-agnostic pseudo-masks, (ii) Pointing Game~\citep{zhang2018top} accuracy, and (iii) mean IoU of the attribution support against ImageNet ground-truth bounding boxes. Table~\ref{tab:localization} summarises the three metrics jointly. SIG attains the best mean SAM~2 IoU and the best Pointing Game accuracy on all three classifiers, and is best on two of three classifiers for BBox IoU (second-best on VGG16). These non-perturbation results are consistent with the perturbation-based and leakage-aware findings, providing a triangulated empirical case for SIG's improved faithfulness.

\begin{table}[ht]
\centering
\caption{\textbf{Localization metrics on ImageNet: SAM~2 IoU, Pointing Game accuracy, and BBox IoU.} Best per column in \textbf{bold}, second-best \underline{underlined}.}
\label{tab:localization}
\vspace{-0.1cm}
\small
\setlength{\tabcolsep}{0.5pt}
\begin{tabular}{l ccc ccc ccc}
\toprule
& \multicolumn{3}{c}{\textbf{SAM~2 IoU} ($\uparrow$)} & \multicolumn{3}{c}{\textbf{Pointing Game} ($\uparrow$)} & \multicolumn{3}{c}{\textbf{BBox IoU} ($\uparrow$)} \\
\cmidrule(lr){2-4} \cmidrule(lr){5-7} \cmidrule(lr){8-10}
\textbf{Method} & R18 & VGG & Inc. & R18 & VGG & Inc. & R18 & VGG & Inc. \\
\midrule
G$\times$I & 0.1582 & 0.1880 & 0.1765 & 0.4909 & 0.5276 & 0.5056 & 0.4902 & 0.4667 & 0.4902 \\
IG  & 0.1804 & 0.2124 & 0.1971 & 0.5576 & 0.6326 & 0.5611 & 0.4936 & 0.4645 & 0.4854 \\
BIG & 0.2084 & 0.2273 & 0.2155 & \underline{0.6576} & \underline{0.6602} & 0.6250 & 0.4993 & 0.4661 & 0.4855 \\
EIG & 0.1727 & 0.1972 & 0.1914 & 0.5576 & 0.5635 & 0.5611 & 0.4877 & 0.4634 & 0.4860 \\
MIG & 0.1756 & 0.2034 & 0.1953 & 0.5545 & 0.5829 & 0.6083 & 0.4943 & \underline{0.4718} & 0.4868 \\
GIG & 0.1880 & 0.1948 & 0.2142 & 0.6061 & 0.5829 & 0.6222 & 0.4589 & 0.4305 & 0.4217 \\
AGI & 0.1750 & 0.2079 & 0.1990 & 0.5576 & 0.6188 & 0.5667 & \underline{0.5054} & \textbf{0.4818} & \underline{0.5024} \\
IG$^2$ & 0.1835 & \underline{0.2312} & 0.2094 & 0.5364 & 0.6298 & 0.5361 & 0.4866 & 0.4463 & 0.4743 \\
SAMP & 0.1767 & 0.1903 & 0.1810 & 0.5091 & 0.5193 & 0.5667 & 0.3929 & 0.3418 & 0.3504 \\
\rowcolor{lightgray}
SIG & \textbf{0.2197} & \textbf{0.2426} & \textbf{0.2322} & \textbf{0.6758} & \textbf{0.6630} & \textbf{0.6694} & \textbf{0.5079} & 0.4684 & \textbf{0.5010} \\
\bottomrule
\end{tabular}
\vspace{-0.2cm}
\end{table}



\section{Generalization Beyond Images}
\label{app:text_domain}

To probe the scope of the coarse-to-fine inductive bias, we apply SIG to text classification by treating token-embedding matrices as the input difference and decomposing them with SVD. We evaluate on SST-2 (DistilBERT and BERT-base) and IMDb (BERT-base and BERT-large) using $500$ test examples. Table~\ref{tab:text_domain} reports DiffID for SIG at several overlap values $\omega$ against IG, $G\times I$, and BIG. With high $\omega$ (close to linear interpolation, i.e., close to IG), SIG approaches or matches IG. With low $\omega$ (strict singular-component ordering), DiffID degrades on three of four settings. This is consistent with the observation that SVD's variance ordering of token-embedding matrices does not, in general, correspond to a semantically meaningful coarse-to-fine progression, and supports our positioning of SIG as primarily an image-domain method.

\begin{table}[ht]
\centering
\vspace{-0.2cm}
\caption{\textbf{Text-domain DiffID ($\uparrow$, $N=500$).} Higher $\omega$ recovers near-IG behavior; lower $\omega$ underperforms.}
\label{tab:text_domain}
\vspace{-0.1cm}
\setlength{\tabcolsep}{2pt}
\begin{tabular}{l cc cc}
\toprule
& \multicolumn{2}{c}{\textbf{SST-2}} & \multicolumn{2}{c}{\textbf{IMDb}} \\
\cmidrule(lr){2-3} \cmidrule(lr){4-5}
\textbf{Method} & DistilBERT & BERT-base & BERT-base & BERT-large \\
\midrule
G$\times$I & 0.2382 & 0.2383 & 0.2152 & \textbf{0.1947} \\
IG  & \textbf{0.2720} & \underline{0.2676} & \textbf{0.2276} & 0.1616 \\
BIG & 0.2290 & 0.2365 & 0.1676 & 0.1547 \\
\midrule
\rowcolor{lightgray}
SIG $\omega=0.1$ & 0.2140 & 0.2251 & 0.1590 & \underline{0.1752} \\
\rowcolor{lightgray}
SIG $\omega=0.4$ & 0.2227 & 0.2445 & 0.1551 & 0.1712 \\
\rowcolor{lightgray}
SIG $\omega=0.7$ & 0.2593 & 0.2671 & 0.2057 & 0.1696 \\
\rowcolor{lightgray}
SIG $\omega=0.9$ & \underline{0.2713} & \textbf{0.2688} & \underline{0.2226} & 0.1642 \\
\bottomrule
\end{tabular}
\vspace{-0.2cm}
\end{table}

\section{Implementation Code}
\label{app:code}
We provide the PyTorch implementation of our proposed method, Spectral Integrated Gradients (SIG). The implementation consists of the core path generation based on Singular Value Decomposition (SVD) and the attribution computation via path integration. 
The source code of this paper has been made publicly available at 
\sigdoi. 
The GitHub repository is available at \sigcode.
\vspace{0.2cm}

\begingroup
\fvset{fontsize=\footnotesize,baselinestretch=0.9}%
\input{mintedcache/listing1.pygtex}%
\endgroup

\vspace{-0.4cm}

\begin{table*}[ht]
\centering
\caption{\textbf{Quantitative comparison on Oxford-IIIT Pet, Oxford 102 Flower, and ImageNet datasets using ResNet18, VGG16, and InceptionV3 classifiers.} Best results are highlighted in \textbf{Bold}, and second-best results are \underline{underlined}.}
\label{tab:quant_full}
\resizebox{\textwidth}{!}{%
\setlength{\tabcolsep}{7pt}
\begin{tabular}{ll rrr rrr rrr rrr}
\toprule
& & \multicolumn{3}{c}{\textbf{ResNet18}} & \multicolumn{3}{c}{\textbf{VGG16}} & \multicolumn{3}{c}{\textbf{InceptionV3}} \\
\cmidrule(lr){3-5} \cmidrule(lr){6-8} \cmidrule(lr){9-11} 
\textbf{Data} & \textbf{Method} & \textbf{DiffID ($\uparrow$)} & \textbf{Ins ($\uparrow$)} & \textbf{Del ($\downarrow$)} & \textbf{DiffID ($\uparrow$)} & \textbf{Ins ($\uparrow$)} & \textbf{Del ($\downarrow$)} & \textbf{DiffID ($\uparrow$)} & \textbf{Ins ($\uparrow$)} & \textbf{Del ($\downarrow$)} \\
\midrule
\rowcolor{palegray}
\cellcolor{white} & \textbf{\textbf{G $\times$ I}} \cite{shrikumar2016not}    & 0.2384 & 0.4378 & 0.1994 & 0.4060 & 0.5174 & 0.1114 & 0.2255 & 0.3940 & 0.1685 \\ 
\cellcolor{white} & \textbf{IG} \cite{sundararajan2017axiomatic}     & 0.3790 & 0.5186 & 0.1396 & 0.5258 & 0.6060 & 0.0802 & 0.3438 & 0.4748 & 0.1309 \\ 
\rowcolor{palegray}
\cellcolor{white} & \textbf{BIG} \cite{xu2020attribution}            & \textbf{0.5462} & \underline{0.6228} & \underline{0.0766} & 0.6252 & 0.6910 & \underline{0.0658} & 0.5246 & 0.5922 & \textbf{0.0676} \\ 
\cellcolor{white} & \textbf{EIG} \cite{jha2020enhanced}              & 0.3417 & 0.4946 & 0.1529 & 0.4853 & 0.5751 & 0.0898 & 0.3228 & 0.4646 & 0.1417 \\ 
\rowcolor{palegray}
\cellcolor{white} & \textbf{MIG} \cite{zaher2024manifold}            & 0.3586 & 0.4982 & 0.1396 & 0.5009 & 0.5829 & 0.0820 & 0.3273 & 0.4652 & 0.1378 \\ 
\cellcolor{white} & \textbf{GIG} \cite{kapishnikov2021guided}        & 0.3634 & 0.5093 & 0.1459 & 0.4709 & 0.5550 & 0.0841 & 0.3586 & 0.4880 & 0.1294 \\ 
\rowcolor{palegray}
\cellcolor{white} & \textbf{AGI} \cite{pan2021explaining}            & 0.2787 & 0.4453 & 0.1667 & 0.4447 & 0.5345 & 0.0898 & 0.3381 & 0.4589 & 0.1207 \\ 
\cellcolor{white} & \textbf{IG$^2$} \cite{zhuo2024ig2}               & 0.3823 & 0.5264 & 0.1441 & 0.6060 & 0.6832 & 0.0772 & 0.4273 & 0.5315 & 0.1042 \\ 
\rowcolor{palegray}
\cellcolor{white} & \textbf{SAMP} \cite{zhang2024path}               & 0.4655 & 0.5315 & \textbf{0.0661} & 0.5357 & 0.6162 & 0.0805 & 0.4751 & 0.5655 & 0.0904 \\ 
\cmidrule(lr){2-11}
\rowcolor{lightgray}
\cellcolor{white} & \textbf{SIG} ($\omega$=0.1)              & 0.4667 & 0.5592 & 0.0925 & 0.5973 & 0.6625 & \textbf{0.0652} & 0.4336 & 0.5132 & 0.0796 \\ 
\rowcolor{lightgray}
\cellcolor{white} & \textbf{SIG} ($\omega$=0.2)              & 0.5114 & 0.6033 & 0.0919 & 0.6357 & 0.7042 & 0.0685 & 0.5042 & 0.5847 & 0.0805\\ 
\rowcolor{lightgray}
\cellcolor{white} & \textbf{SIG} ($\omega$=0.3)              & 0.5117 & 0.6099 & 0.0982 & 0.6441 & 0.7189 & 0.0748 & 0.5192 & 0.6009 & 0.0817 \\ 
\rowcolor{lightgray}
\cellcolor{white} & \textbf{SIG} ($\omega$=0.4)              & \underline{0.5210} & \textbf{0.6237} & 0.1027 & \textbf{0.6532} & \textbf{0.7270} & 0.0739 & \underline{0.5330} & \underline{0.6114} & \underline{0.0784} \\ 
\rowcolor{lightgray}
\cellcolor{white} & \textbf{SIG} ($\omega$=0.5)              & 0.5180 & 0.6201 & 0.1021 & \underline{0.6529} & \underline{0.7261} & 0.0733 & \textbf{0.5345} & \textbf{0.6147} & 0.0802 \\ 
\rowcolor{lightgray}
    \multirow{-15}{*}{\rotatebox[origin=c]{90}{\textbf{Oxford-IIIT Pet}}} 
\cellcolor{white} & \textbf{SIG} ($\omega$=0.6)              & 0.3631 & 0.5045 & 0.1414 & 0.5258 & 0.6024 & 0.0766 & 0.3904 & 0.5063 & 0.1159 \\ 
\midrule
\cellcolor{white} & \textbf{G $\times$ I} \cite{shrikumar2016not}    & 0.1222 & 0.2338 & 0.1116 & 0.2784 & 0.3576 & 0.0791 & 0.2000 & 0.2813 & 0.0813 \\ 
\rowcolor{palegray}
\cellcolor{white} & \textbf{IG} \cite{sundararajan2017axiomatic}     & 0.1769 & 0.2740 & 0.0971 & 0.3178 & 0.3847 & 0.0669 & 0.2551 & 0.3247 & 0.0696 \\ 
\cellcolor{white} & \textbf{BIG} \cite{xu2020attribution}            & \underline{0.2629} & \underline{0.3131} & \textbf{0.0502} & 0.3116 & 0.3711 & \textbf{0.0596} & 0.2556 & 0.3140 & 0.0584 \\ 
\rowcolor{palegray}
\cellcolor{white} & \textbf{EIG} \cite{jha2020enhanced}              & 0.1871 & 0.2842 & 0.0971 & 0.3218 & 0.3860 & 0.0642 & 0.2553 & 0.3249 & 0.0696 \\ 
\cellcolor{white} & \textbf{MIG} \cite{zaher2024manifold}            & 0.1738 & 0.2729 & 0.0991 & 0.3076 & 0.3751 & 0.0676 & 0.2531 & 0.3260 & 0.0729 \\ 
\rowcolor{palegray}
\cellcolor{white} & \textbf{GIG} \cite{kapishnikov2021guided}        & 0.1891 & 0.272 & 0.0829 & 0.2587 & 0.3449 & 0.0862 & 0.2611 & \textbf{0.3413} & 0.0802  \\
\cellcolor{white} & \textbf{AGI} \cite{pan2021explaining}            & 0.0136 & 0.1569 & 0.1433 & 0.1409 & 0.2689 & 0.1280 & 0.0787 & 0.1947 & 0.1160 \\ 
\rowcolor{palegray}
\cellcolor{white} & \textbf{IG$^2$} \cite{zhuo2024ig2}               & 0.0193 & 0.1713 & 0.1520 & 0.2189 & 0.3293 & 0.1104 & 0.0816 & 0.2053 & 0.1238 \\ 
\cellcolor{white} & \textbf{SAMP} \cite{zhang2024path}               & \textbf{0.2751} & \textbf{0.3498} & 0.0747 & 0.2869 & 0.3638 & 0.0769 & 0.2662 & \underline{0.3380} & 0.0718 \\ 
\cmidrule(lr){2-11}
\rowcolor{lightgray}
\cellcolor{white} & \textbf{SIG} ($\omega$=0.1)              & 0.2276 & 0.2831 & 0.0556 & 0.2996 & 0.3644 & 0.0649 & 0.2427 & 0.2971 & 0.0544 \\ 
\rowcolor{lightgray}
\cellcolor{white} & \textbf{SIG} ($\omega$=0.2)              & 0.2493 & 0.3018 & \underline{0.0524} & 0.3173 & 0.3802 & 0.0629 & 0.2653 & 0.3167 & 0.0513 \\ 
\rowcolor{lightgray}
\cellcolor{white} & \textbf{SIG} ($\omega$=0.3)              & 0.2513 & 0.3053 & 0.0540 & \textbf{0.3331} & \underline{0.3942} & \underline{0.0611} & \underline{0.2753} & 0.3229 & \textbf{0.0476} \\ 
\rowcolor{lightgray}
\cellcolor{white} & \textbf{SIG} ($\omega$=0.4)              & \underline{0.2533} & 0.3116 & 0.0582 & 0.3224 & 0.3856 & 0.0631 & \textbf{0.2822} & \underline{0.3316} & \underline{0.0493} \\ 
\rowcolor{lightgray}
\cellcolor{white} & \textbf{SIG} ($\omega$=0.5)              & 0.2471 & 0.3071 & 0.0600 & 0.3136 & 0.3780 & 0.0644 & 0.2751 & 0.3284 & 0.0533 \\ 
\rowcolor{lightgray}
    \multirow{-15}{*}{\rotatebox[origin=c]{90}{\textbf{Oxford 102 Flower}}} 
\cellcolor{white} & \textbf{SIG} ($\omega$=0.6)              & 0.2136 & 0.2884 & 0.0749 & \underline{0.3311} & \textbf{0.3947} & 0.0636 & 0.2627 & 0.3176 & 0.0549 \\ 
\midrule
\rowcolor{palegray}
\cellcolor{white} & \textbf{G $\times$ I} \cite{shrikumar2016not}    & 0.1038 & 0.2278 & 0.1240 & 0.1882 & 0.2749 & 0.0867 & 0.1380 & 0.2776 & 0.1396 \\ 
\cellcolor{white} & \textbf{IG} \cite{sundararajan2017axiomatic}     & 0.1589 & 0.2560 & 0.0971 & 0.2404 & 0.3093 & 0.0689 & 0.1916 & 0.3073 & 0.1158 \\ 
\rowcolor{palegray}
\cellcolor{white} & \textbf{BIG} \cite{xu2020attribution}            & 0.3122 & 0.3504 & 0.0382 & 0.3258 & 0.3549 & \textbf{0.0291} & 0.3473 & 0.3956 & 0.0482 \\ 
\cellcolor{white} & \textbf{EIG} \cite{jha2020enhanced}              & 0.1500 & 0.2533 & 0.1033 & 0.2327 & 0.3047 & 0.0720 & 0.1911 & 0.3091 & 0.1180 \\ 
\rowcolor{palegray}
\cellcolor{white} & \textbf{MIG} \cite{zaher2024manifold}            & 0.1596 & 0.2556 & 0.096 & 0.2349 & 0.3040 & 0.0691 & 0.1940 & 0.3104 & 0.1164 \\ 
\cellcolor{white} & \textbf{GIG} \cite{kapishnikov2021guided}        & 0.2507 & 0.3142 & 0.0636 & 0.2813 & 0.3302 & 0.0489 & 0.3027 & 0.3789 & 0.0762 \\ 
\rowcolor{palegray}
\cellcolor{white} & \textbf{AGI} \cite{pan2021explaining}            & 0.1640 & 0.2553 & 0.0913 & 0.2260 & 0.2922 & 0.0662 & 0.2171 & 0.3096 & 0.0924 \\ 
\cellcolor{white} & \textbf{IG$^2$} \cite{zhuo2024ig2}               & 0.1824 & 0.2676 & 0.0851 & 0.3047 & 0.3584 & 0.0538 & 0.2569 & 0.3447 & 0.0878 \\ 
\rowcolor{palegray}
\cellcolor{white} & \textbf{SAMP} \cite{zhang2024path}               & 0.2880 & 0.3278 & 0.0398 & 0.2811 & 0.3118 & 0.0307 & 0.3089 & 0.3751 & 0.0662 \\ 
\cmidrule(lr){2-11}
\rowcolor{lightgray}
\cellcolor{white} & \textbf{SIG} ($\omega$=0.1)              & 0.3151 & 0.3478 & \textbf{0.0327} & 0.3464 & 0.3767 & \underline{0.0302} & 0.3342 & 0.3813 & \underline{0.0471} \\ 
\rowcolor{lightgray}
\cellcolor{white} & \textbf{SIG} ($\omega$=0.2)              & 0.3307 & 0.3669 & 0.0362 & 0.3647 & 0.3980 & 0.0333 & 0.3680 & 0.4100 & \textbf{0.0420} \\ 
\rowcolor{lightgray}
\cellcolor{white} & \textbf{SIG} ($\omega$=0.3)              & \textbf{0.3429} & 0.3787 & \underline{0.0358} & \textbf{0.3736} & \underline{0.4093} & 0.0358 & \underline{0.3807} & \underline{0.4278} & \underline{0.0471} \\ 
\rowcolor{lightgray}
\cellcolor{white} & \textbf{SIG} ($\omega$=0.4)              & \underline{0.3422} & \textbf{0.3800} & 0.0378 & \textbf{0.3736} & \textbf{0.4107} & 0.0371 & 0.3784 & 0.4264 & 0.0480 \\ 
\rowcolor{lightgray}
\cellcolor{white} & \textbf{SIG} ($\omega$=0.5)              & \underline{0.3422} & \underline{0.3789} & 0.0367 & \underline{0.3691} & 0.4091 & 0.0400 & \textbf{0.3836} & \textbf{0.4309} & 0.0473 \\ 
\rowcolor{lightgray}
    \multirow{-15}{*}{\rotatebox[origin=c]{90}{\textbf{ImageNet2012}}} 
\cellcolor{white} & \textbf{SIG} ($\omega$=0.6)              & 0.2302 & 0.2998 & 0.0696 & 0.2858 & 0.3440 & 0.0582 & 0.3116 & 0.3762 & 0.0647 \\ 
\bottomrule
\end{tabular}
}
\end{table*}

\begin{figure*}[ht]
    \centering
    \begin{subfigure}[b]{0.48\linewidth}
        \centering
        \includegraphics[width=\linewidth]{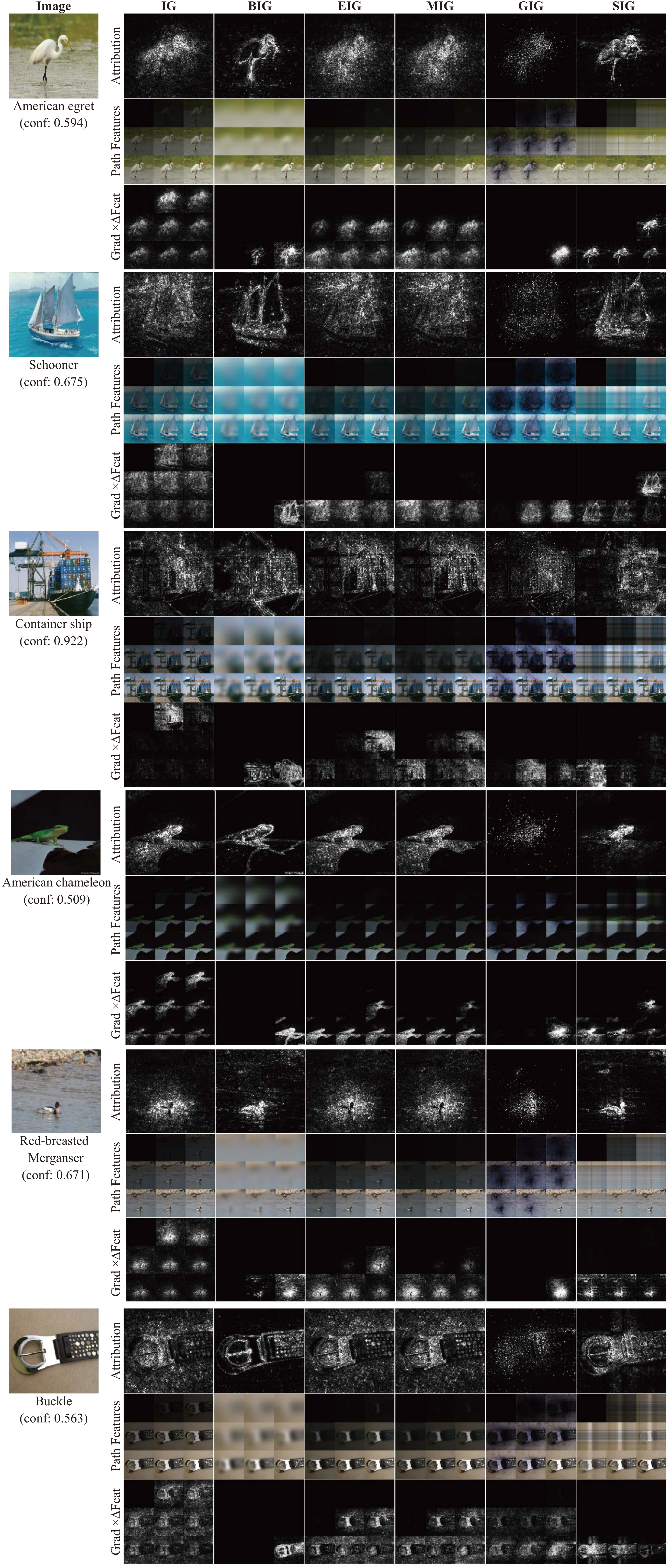}
        \caption{ImageNet2012}
    \end{subfigure}
    \hfill 
    \begin{subfigure}[b]{0.48\linewidth}
        \centering
        \includegraphics[width=\linewidth]{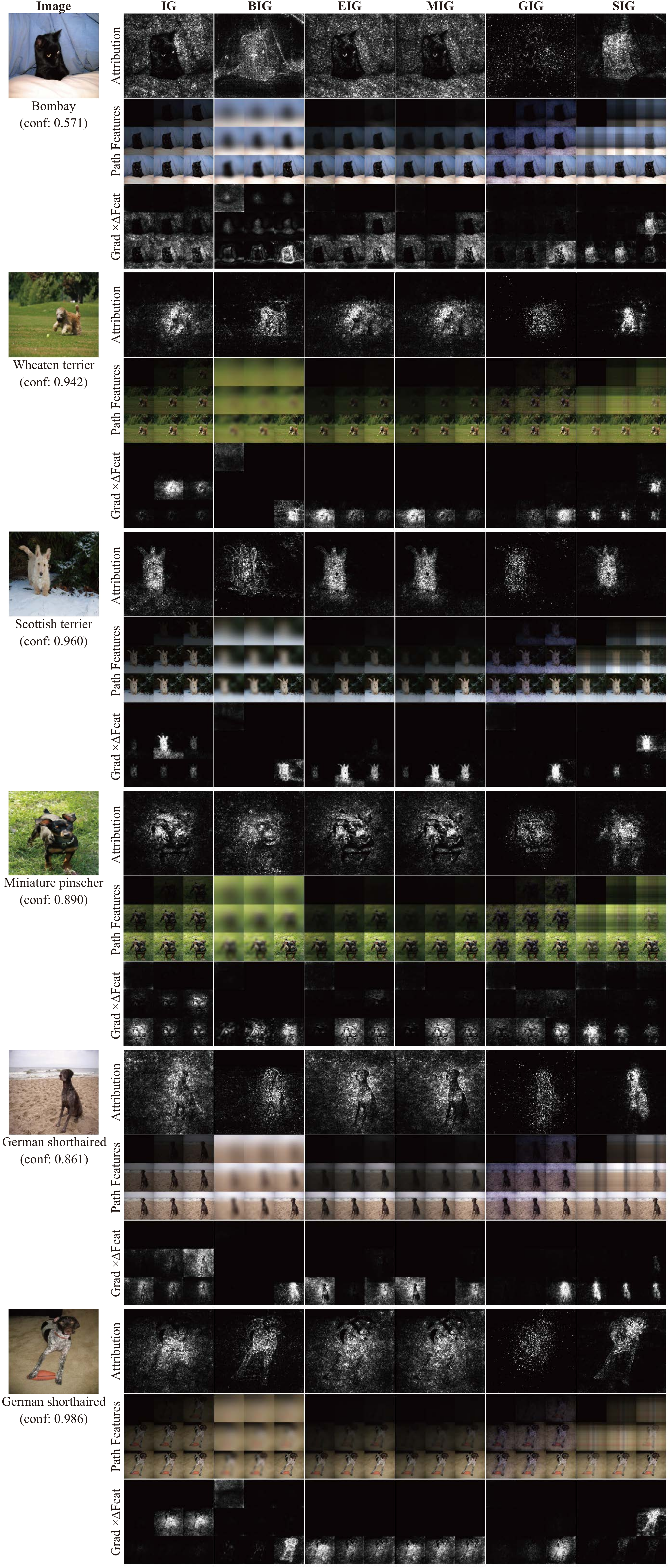}
        \caption{Oxford IIIT Pet}
    \end{subfigure}
    \caption{\textbf{Extended path analysis comparing IG, BIG, EIG, MIG, GIG, and SIG.} For each example, the top row displays the final attribution maps. The second and third rows visualize the evolution of path features and their corresponding gradients multiplied by the path difference, sampled at nine equally spaced intervals along the integration path. This demonstrates how SIG aggregates relevant attributions in a coarse-to-fine manner. (Conf.: Confidence)}
    \label{fig:path_image}
    \Description{Two side by side extended path analysis panels on ImageNet2012 (left) and Oxford IIIT Pet (right), comparing six attribution methods: IG, BIG, EIG, MIG, GIG, and SIG. For each example, the top row shows the final attribution maps of all six methods. The second and third rows display the path features and gradient contributions sampled at nine equally spaced steps along the integration path.}
\end{figure*}

\begin{figure*}[ht]
    \centering
    \includegraphics[width=\linewidth]{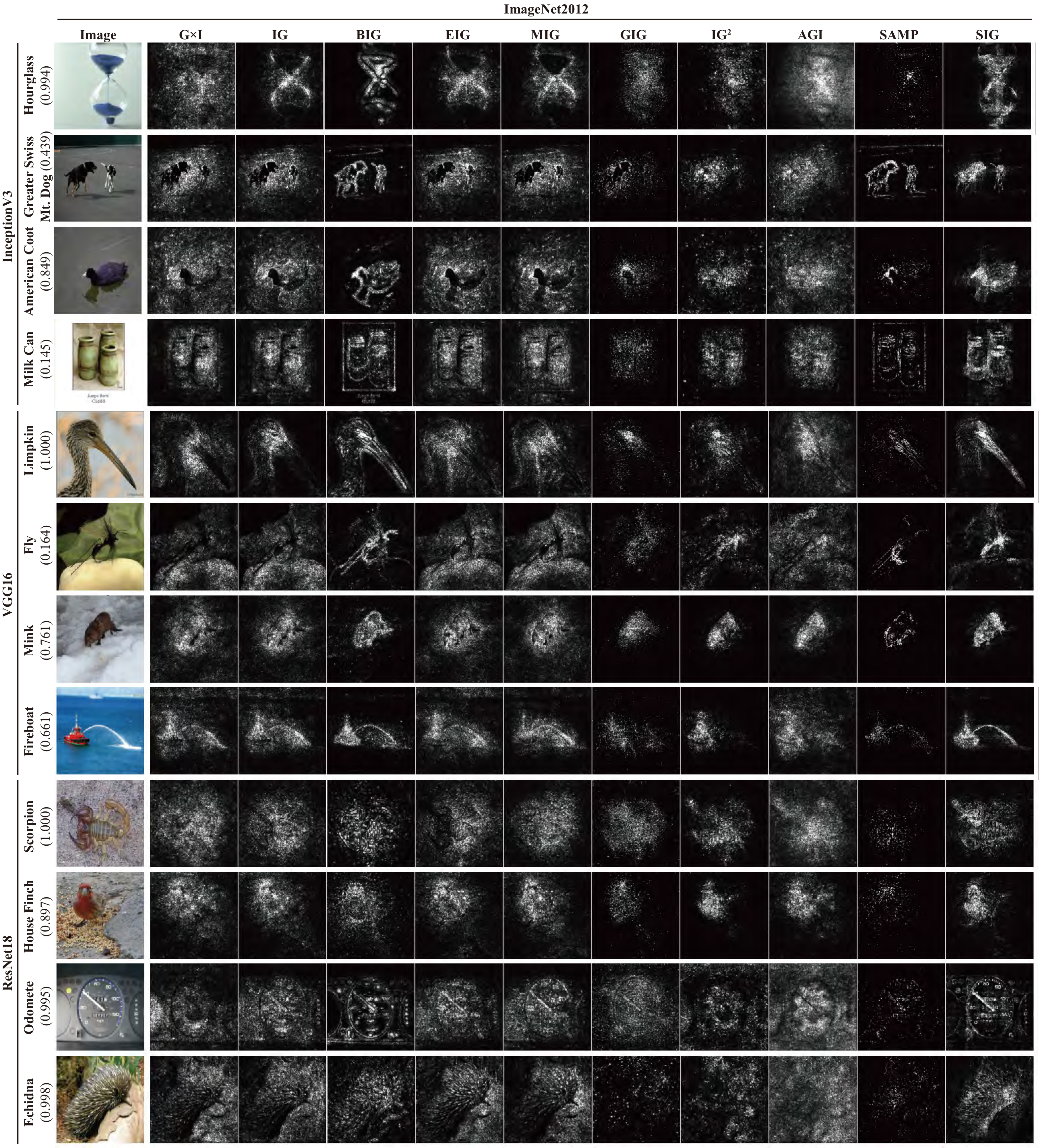}
    \caption{\textbf{Qualitative comparison against baselines on ImageNet 2012 using three classifiers.} Bold test labels indicate the predicted class, and numbers in brackets denote confidence.}
    \Description{Grid of attribution maps on ImageNet 2012 comparing ten methods across ResNet18, VGG16, and InceptionV1.}
    \label{fig:qual_imagenet}
\end{figure*}

\begin{figure*}[ht]
    \centering
    \includegraphics[width=\linewidth]{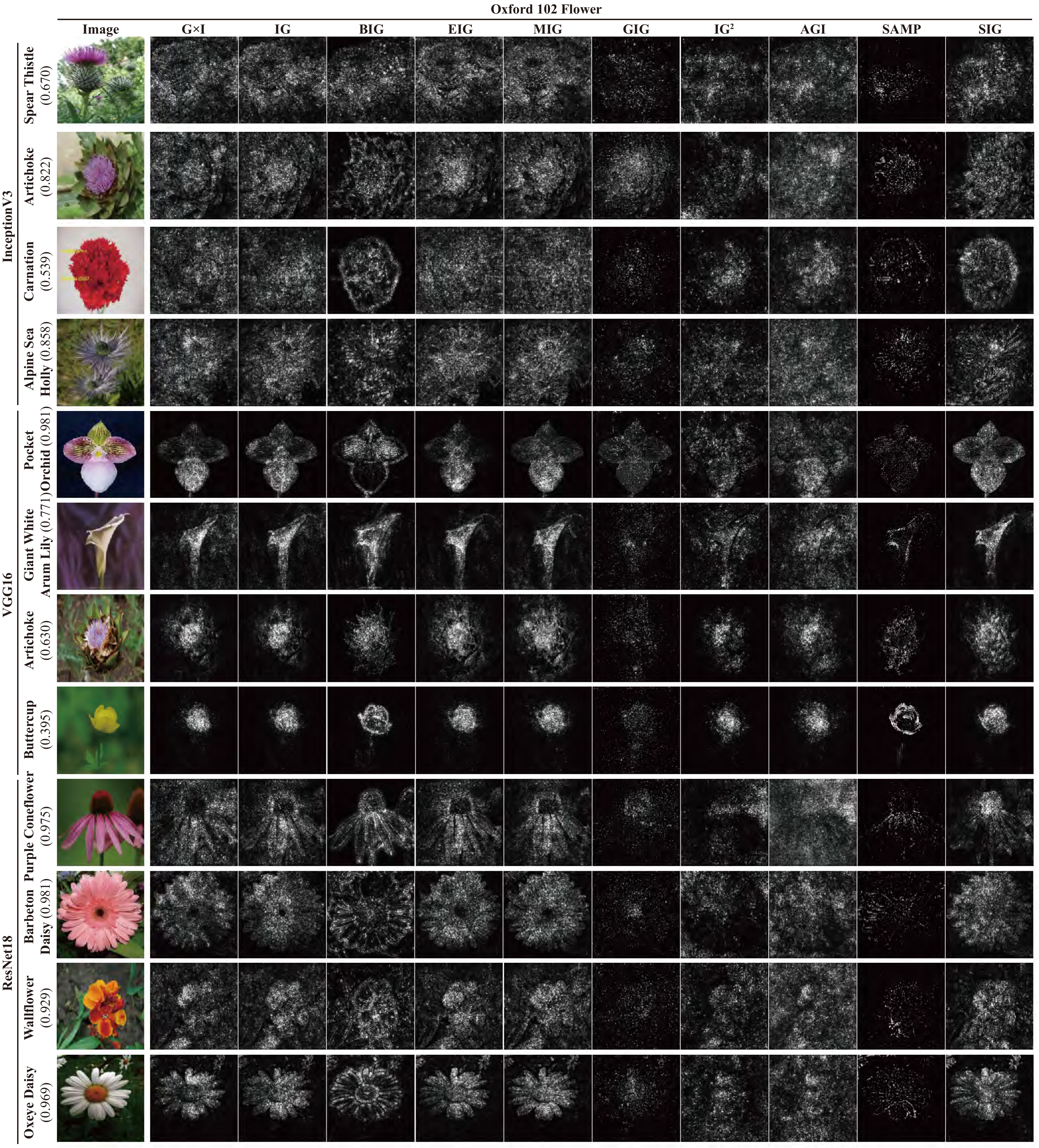}
    \caption{\textbf{Qualitative comparison on Oxford 102 Flower using three classifiers.} Bold test labels indicate the predicted class, and numbers in brackets denote confidence.}
    \Description{Grid of attribution maps on Oxford 102 Flower comparing ten methods across ResNet18, VGG16, and InceptionV1. SIG highlights petal and stamen regions while suppressing background clutter more effectively than baselines.}
    \label{fig:qual_oxfordflower}
\end{figure*}

\begin{figure*}[ht]
    \centering
    \includegraphics[width=\linewidth]{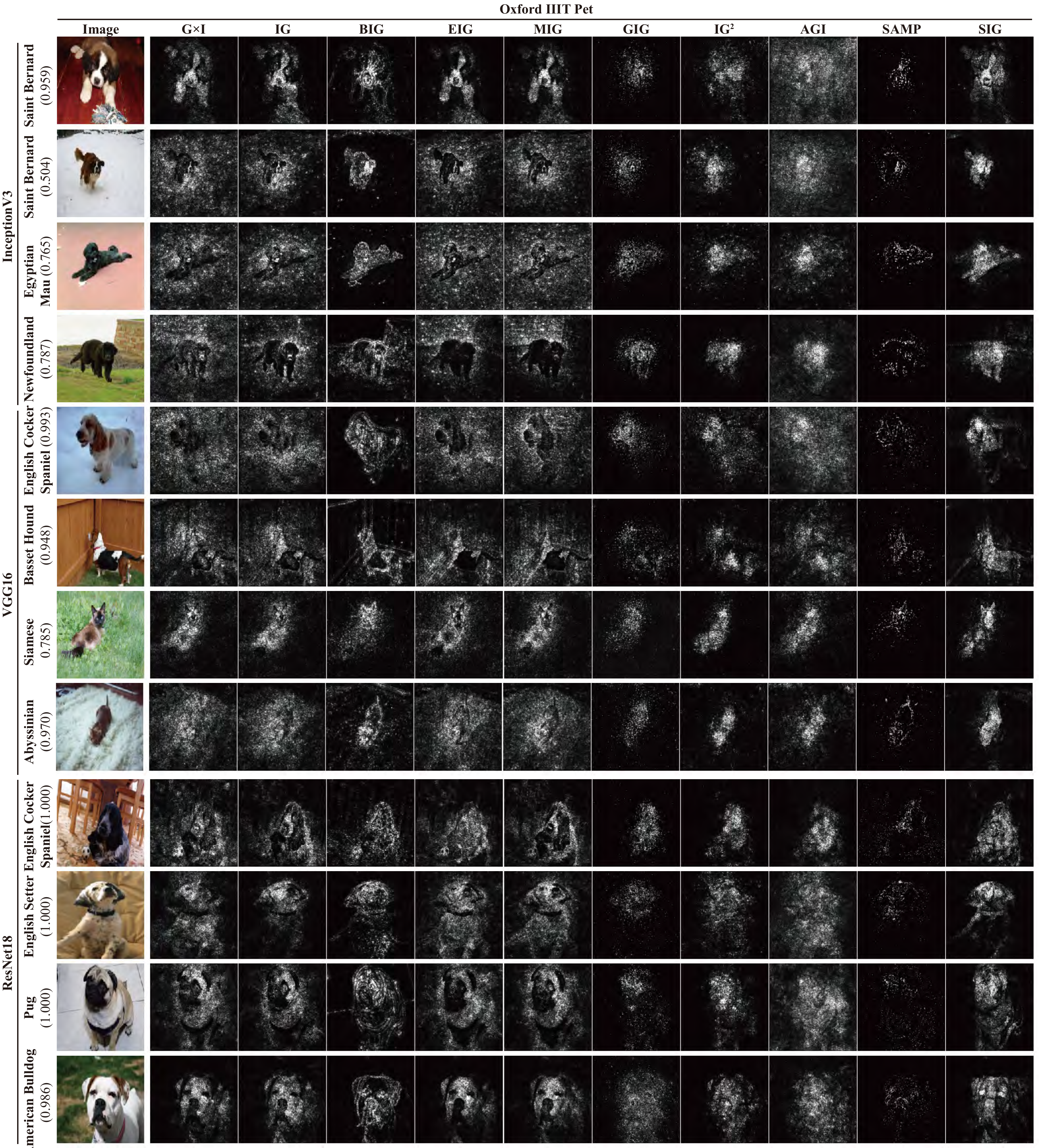}
    \caption{\textbf{Qualitative comparison on Oxford-IIIT Pet using three classifiers.} Bold test labels indicate the predicted class, and numbers in brackets denote confidence.}
    \Description{Grid of attribution maps on Oxford IIIT Pet comparing ten methods across ResNet18, VGG16, and InceptionV1. SIG concentrates attribution on facial and body features of each pet while other methods exhibit scattered patterns.}
    \label{fig:qual_oxfordpet}
\end{figure*}

\end{document}